\documentclass[runningheads]{llncs}

\usepackage{eccv}

\usepackage{eccvabbrv}

\usepackage{graphicx}
\usepackage{booktabs}
\usepackage{multirow}
\usepackage[export]{adjustbox}

\usepackage{tcolorbox}
\tcbuselibrary{listings,skins,breakable}

\usepackage{expl3}

\usepackage{arydshln}
\usepackage{amsmath}
\usepackage{amssymb}

\usepackage[accsupp]{axessibility}  

\usepackage[hidelinks]{hyperref}

\usepackage[all]{hypcap}
\usepackage[svgnames,table]{xcolor}

\usepackage{orcidlink}

\newcommand{\beq}{\vspace{0mm}\begin{equation}}
\newcommand{\eeq}{\vspace{0mm}\end{equation}}
\newcommand{\beqs}{\vspace{0mm}\begin{eqnarray}}
\newcommand{\eeqs}{\vspace{0mm}\end{eqnarray}}
\newcommand{\barr}{\begin{array}}
\newcommand{\earr}{\end{array}}

\newcommand{\ones}[0]{\mathbb{I}}

\usepackage{booktabs}
\usepackage{array}
\newcolumntype{L}[1]{>{\raggedright\let\newline\\\arraybackslash\hspace{0pt}}m{#1}}
\newcolumntype{C}[1]{>{\centering\let\newline\\\arraybackslash\hspace{0pt}}m{#1}}
\newcolumntype{R}[1]{>{\raggedleft\let\newline\\\arraybackslash\hspace{0pt}}m{#1}}

\newcommand{\D}{\mathcal{D}}

\newcommand{\name}{\textsc{Spire}}

\newcommand{\de}[2]{%
  \delta_{#1}%
  \if\relax\detokenize{#2}\relax%
  \else(#2)%
  \fi%
}

\definecolor{lightgreen}{rgb}{0.22,0.70,0.30}%
\definecolor{Gray}{gray}{0.95}
\definecolor{Cornsilk}{rgb}{1.0, 0.97, 0.86}

\definecolor{lightblue}{HTML}{0064E0}
\definecolor{fg}{HTML}{1C2B33}
\definecolor{bg}{HTML}{F1F4F7}

\newtheorem{assumption}[theorem]{Assumption}

\begin{document}

\title{
Personalization as Inverse Planning: Learning Latent Design Intents for Agentic Slide Generation via Structural Denoising
} 

\titlerunning{
}

\author{
Tianci Liu\inst{1,*} \and
Zihan Dong\inst{2} \and
Linjun Zhang\inst{2} \and
Haoyu Wang\inst{3} \and
Jing Gao\inst{1} \and \\
Emre K\i c\i man\inst{4} \and
Ranveer Chandra\inst{4} \and
Wei-Ting Chen\inst{4}\textsuperscript{\dag}
}

\authorrunning{T.~Liu et al.}

\institute{
Purdue University \and
Rutgers University \and
University at Albany \and
Microsoft
}

\maketitle

\begin{abstract}

Slide design requires personalizing both deck themes and page layouts. Yet, current AI agent-based methods struggle with fine-grained, page-level design. Solely relying on prespecified templates or user verbose instructions, they fail to capture latent design intents, leaving Page-level Slide Personalization (PSP) unresolved. To close this gap, this work formulates PSP as an inverse planning problem. We propose to learn a design intent without assuming any knowledge of the specific executing tools (e.g., PowerPoint, Beamer) being used. However, relinquishing control over these tools makes the problem intractable to optimize end-to-end. To overcome this, we propose {\name}, a principled framework to solve PSP approximately. By intentionally corrupting the visual structures of clean slides, {\name} creates a verifiable task to denoise the corruption, whereby two agents learn to collaboratively refine executable designs via reinforcement learning (RL). We present a proof that structural denoising is a consistent surrogate for PSP, and that the multi-agent formulation strictly reduces policy gradient variance in RL. Extensive experiments demonstrate the superiority of {\name}.

\end{abstract}

\def\thefootnote{*}\footnotetext{Work done during an internship at Microsoft.}\def\thefootnote{\arabic{footnote}}
\def\thefootnote{†}\footnotetext{Corresponding author.}\def\thefootnote{\arabic{footnote}}
\vspace{3mm}

\section{Introduction}
\label{sec:intro}

\begin{figure*}[t!]
\centering
\includegraphics[width=1.0\linewidth]{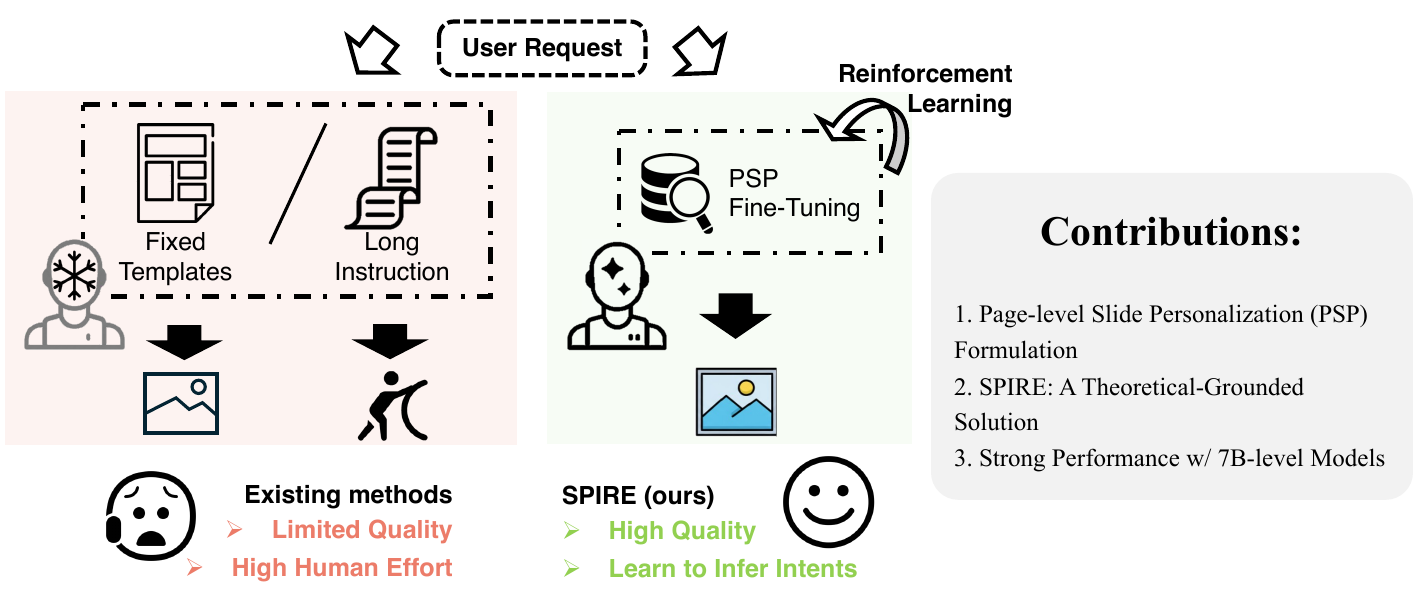}
\caption{
Illustration of {\name}.
Trained with reinforcement learning, 
{\name} learns to infer design intents instead of using prespecified layout templates or lengthy user instructions as existing methods do~\cite{zheng2025pptagent,ge2025autopresent}, offering both \emph{theoretical} and \emph{empirical} advantages.
}
\label{fig:teaser}
\end{figure*}

Slide decks are a primary medium for communicating ideas in academia and industry~\cite{rojo2006critical,hu2013ppsgen,trelease2016chalkboard}. 
Yet, creating a high-quality deck is a design-intensive process. Depending on the presenter and the target audience, the \emph{same} content can call for different visual treatments. 
These treatments are often personalized to a speaker's habitual design, a lab's visual identity, or an organization's brand guidelines. 
In short, practical slide generation is inherently a \emph{personalized} task.

Recent advances in multi-modal large language models (MLLMs) have enabled agentic pipelines that automate slide generation for practical use~\cite{sun2021d2s,fu2022doc2ppt,mondal2024presentations,ge2025autopresent,zheng2025pptagent,xu2025textual,pang2025paper}.
These systems typically decompose slide creation into modular stages such as outlining, asset extraction, layout arrangement, and iterative refinement, thereby improving deck-level coherence via template selection, visual feedback, and multi-agent coordination~\cite{zheng2025pptagent,zeng2025slidetailor,pang2025paper,xie2025slidebot,jung2025talk,jang2026deckbench}.
However, they struggle with fine-grained \emph{page-level} design, 
{which requires deciding visual hierarchy, element alignment, spacing,
and styling choices.} 
Nonetheless, 
existing agentic systems handle page-level layout and styling passively:
Once the content is specified, 
the page-level design is largely overlooked, either by following generic system-defined templates~\cite{maheshwari2024presentations,zheng2025pptagent,xie2025slidebot,jang2026deckbench}, or by requiring lengthy, explicit user instructions~\cite{ge2025autopresent,pang2025paper}, instead of inferring the user’s latent intent.
Consequently, they are too coarse to achieve satisfactory Page-level Slide Personalization (PSP).

To close this gap, we introduce the task of agentic PSP; the concept is depicted in Fig~\ref{fig:teaser}.
Our formulation is inspired by how human designers teach slide making in practice:
\emph{Given} a detailed plan elaborating the page specification such as layout structure and visual hierarchy,
a generally capable executor (e.g., an experienced designer unfamiliar with specific corporate styles) can reliably reproduce a high-quality personalized slide visual by following the plan step-by-step.

But such actionable plans are infeasible to collect at scale in practice.
{Detailed intent annotations are rarely available
in real-world slide corpora~\cite{ge2025autopresent}}, making 
direct supervised training impracticable.
As a result, PSP requires one to \emph{proactively infer} the plan as a latent intent in order to actively guide the generation process.
Furthermore, this plan is naturally context-dependent: (enterprise) users often maintain more than one set of intents depending on their specific scenarios.
Fortunately, such contextual intent can be largely specified by having the user provide a small number of reference slides.
To this end, we formulate PSP as a probabilistic problem that explicitly models the design intent as a latent random variable.
Given (1) user-provided page assets and (2) a small set of reference slides, we infer the final, detailed plan, which is then seamlessly passed to a downstream visual designer for execution.
Intuitively, our formulated PSP aims to infer a plan that, once executed by a capable, non-personalized executor, can most reliably reproduce the desired personalized slide given the user's references.
{\emph{This probabilistic formulation treats design intent as a latent variable and provides a principled foundation for page-level slide personalization.}}

Crucially, this latent-intent formulation of PSP is computationally challenging to solve for two reasons.
First, it evaluates a proposed plan based on whether it enables the executor to reproduce the desired visual.
This requires a rendering likelihood, which lacks a clear, direct objective that one can optimize.
Naive image-level similarity provides an unreliable estimate for slide quality, which is governed by discrete, structured design decisions rather than purely pixel-wise closeness~\cite{ge2025autopresent,tang2025slidecoder,qu2025igd,xu2025pregenie}.
Second, the executor that turns a plan into a slide is effectively a black box, which further impedes optimizing the PSP objective with standard numerical methods~\cite{ruder2016overview,lan2020first}.
Therefore, an effective approximation is needed to make this objective tractable. We detail this formulation and its computational challenges in Sec.~\ref{sec:method:form}.

As a remedy, in this work we propose \name\ ({\underline{S}tructural \underline{P}lanning via \underline{I}nverse \underline{RE}construction}), a structural denoising framework that turns the hard-to-optimize \emph{slide} personalization objective into a verifiable reconstruction problem that serves as a good proxy.
Our solution extends existing multi-agent visual generation frameworks~\cite{hahn2024proactive,ma2025talk2image,venkatesh2025crea,jang2026deckbench,jaiswal2026iterative} to a principled training objective.
With the aforementioned PSP goal of learning a reference-conditioned page-level plan, we alternate between a critic that provides semantic feedback and a planner that updates its plan accordingly.
The key idea is to exploit the discrete structural nature of slides: instead of perturbing pixels, we intentionally corrupt a gold slide by applying random, element-level structural perturbations to its layout, hierarchy, and styling, and record the corresponding discrepancies.
\emph{This yields scalable self-supervision where suboptimal parts to improve are known by construction.}

Built upon the structural perturbation,
\name\ trains two complementary components separately on this shared signal.
A critic learns to read the corrupted slide and the user's references, and to produce structured, actionable feedback that pinpoints the discrepancies as semantic edit suggestions.
A planner learns to produce an executable, editable plan, either proposing a plan from scratch or revising an imperfect plan guided by the critic's feedback.
As both components are trained to recover from controlled structural corruptions, the critic's feedback becomes verifiable,
and the planner learns to improve plans without relying on gradients through the black-box executor, making latent-intent PSP tractable in practice.
We resort to reinforcement learning~\cite{shao2024deepseekmath,yu2025dapo} to train the two agents.
Sec.~\ref{sec:method:spire} details the proposed {\name}.
We provide theoretical analysis of its superiority in Sec.~\ref{sec:theory}.
\emph{This principled method and its theoretical advantage are the main technical contribution of this work.}

Our paper is organized as follows.
Sec.~\ref{sec:method} formulates PSP and details the proposed {\name}.
Extensive experimental results in Sec.~\ref{sec:exp} demonstrate the effectiveness of our method.
In the remaining part of this paper, we review related work in Sec.~\ref{sec:relate}, and conclude the paper in Sec.~\ref{sec:conclusion}.

\section{Proposed Method}
\label{sec:method}

This section formalizes the task of reference-based
page-level slide personalization (PSP),
which entails an intractable optimization problem due to contextual intent being latent.
We propose {\name} as a principled approximate solution.

\subsection{Problem Formulation}
\label{sec:method:form}

Suppose a (enterprise) user specifies a high-level instruction $x$ about slide content (e.g., text content and visual elements)\footnote{In this paper, we refer to page-level objects on a slide (e.g., text boxes, images, shapes, charts, and tables) uniformly as \emph{visual elements}. A \emph{text box} is a type of visual element that contains \emph{text content}.}, and provides $K$ reference slides
\begin{align*}
\D_\text{ref} = \{ s_\text{ref}^{(1)}, \ldots, s_\text{ref}^{(K)} \},
\end{align*}
to exemplify their contextually preferred style.
We define PSP as producing a target slide $s$ that can accurately reflect the user's exact intent that is not elaborated in the verbal $x$. To reflect the need for creative design, we allow $\D_\text{ref}$ not to cover the optimal layout $s$ should use.

We refer to this intent as a \emph{plan} and consider it as a latent variable denoted by $z$.
Intuitively, $z$ can be a step-by-step actionable instruction that contains a rich page-level specification (e.g., layout coordinates and element styling).

Following a well-formed $z$,
a generally capable executor $E$ can be unambiguously guided to produce a high quality visual $s$ that is \emph{personalized} to the user's needs, even if $E$ itself is not directly tailored to the user's preference.
From a probabilistic perspective,
this implies that the visual realization $s$ is \emph{conditionally independent} of the user context given the plan $z$:
\begin{align*}
p(s \mid z, x, \mathcal{D}_{\text{ref}}) = p(s \mid z).
\end{align*}
Here the randomness is induced by the use of black-box $E$, which can be a coding agent~\cite{openai2022chatgpt,zhang2024codeagent,dong2025survey,comanici2025gemini,Claude4,gpt5,jung2025talk},
an image generator~\cite{ye2023ip,esser2024scaling,ma2024subject,gpt-image-1}, or a human designer.
Note that disentangling the plan $z$ and $E$ is on purpose, so that PSP does not need to be conducted whenever a new executor $E$ is introduced.

Based on this definition,
letting $\pi_\theta$ denote a multi-modal language model \emph{planner},
PSP can be formulated as learning $\pi_\theta$ from a personal corpus
\begin{align*}
\D_\text{tr} = \{(x^{(1)}, (s^{*})^{(1)}, \D_\text{ref}^{(1)}), \ldots, (x^{(J)}, (s^{*})^{(J)}, \D_\text{ref}^{(J)})\}.
\end{align*}
Here $s^*$ denotes the gold slide for $x$. We detail the construction of $\D_\text{tr}$ in the supplementary material.

Given  $\D_\text{tr}$, the goal of PSP is to maximize the marginal likelihood of $s^*$ given $x, \D_\text{ref}$. This  can be expressed as the following optimization objective:
{\small
\begin{align}
\notag
\mathcal{J}(\theta)
&= \mathbb{E}_{(x, s^*, \D_\text{ref}) \sim \mathcal{D}_{\text{tr}}} \big[ \log p_\theta(s^* \mid x, \mathcal{D}_{\text{ref}}) \big] \\
&= \mathbb{E}_{(x, s^*, \D_\text{ref}) \sim \mathcal{D}_{\text{tr}}} \left[ \log \int \pi_\theta(z \mid x, \mathcal{D}_{\text{ref}}) \, p(s^* \mid z) \, dz \right].
\label{eq:formal}
\end{align}
}%
At a colloquial level, Eq.~\eqref{eq:formal} seeks $\pi_\theta$ that can generate plans that can maximize the likelihood of black-box $E$ to reproduce (generate) the gold $s^*$.
Unfortunately, optimizing Eq.~\eqref{eq:formal} is prohibitive due to two reasons.

First, the likelihood $p(s^* \mid z)$ lacks an explicit form, which makes the objective intractable.
One common solution is to approximately optimize
{\small
\begin{align*}
p(s^* \mid z) \propto -\| s^* - E(z) \|_2,
\end{align*}
}%
in the pixel space, aiming to push generated $s = E(z)$ towards $s^*$ pixel-wise.
Yet, recent works~\cite{ge2025autopresent,tang2025slidecoder,xu2025pregenie} showed that treating slides as pure images struggles to capture their discrete logical structure, offering unreliable guidance for the planner.

Second, black-box $E$ cannot be differentiated through,
which prevents direct computation of the gradient for $\pi_\theta$ with respect to $s = E(z)$ given by
\begin{align*}
\nabla_\theta \log p_\theta(s = E(z) \mid x, \mathcal{D}_{\text{ref}}).
\end{align*}
In addition, zeroth-order optimization methods~\cite{chen2023instructzero,malladi2023fine,hu2024localized,li2025survey} \textcolor{purple}{prove} ineffective for this context for two reasons. First, their high computational cost~\cite{park2025zip,zhan2024unlocking,qi2025learning} makes them prohibitive for slide generation, where user preferences and instructions take diverse forms. Second, these methods approximate gradients by probing the specific executor $E$, which makes them prone to overfitting on \textcolor{purple}{a} particular instance~\cite{hu2024localized}, leading to limited generalizability of $\pi_\theta$ to other executors~\cite{qi2025learning}.

In the next section, we show that by leveraging the \emph{structural} discreteness of slide visuals, we can circumvent these hurdles via a surrogate denoising objective, and we propose {\name} as an effective approximate solution.

\subsection{\name: An Effective Solution for PSP}
\label{sec:method:spire}

\begin{figure*}[t!]
\centering
\includegraphics[width=1.05\linewidth]{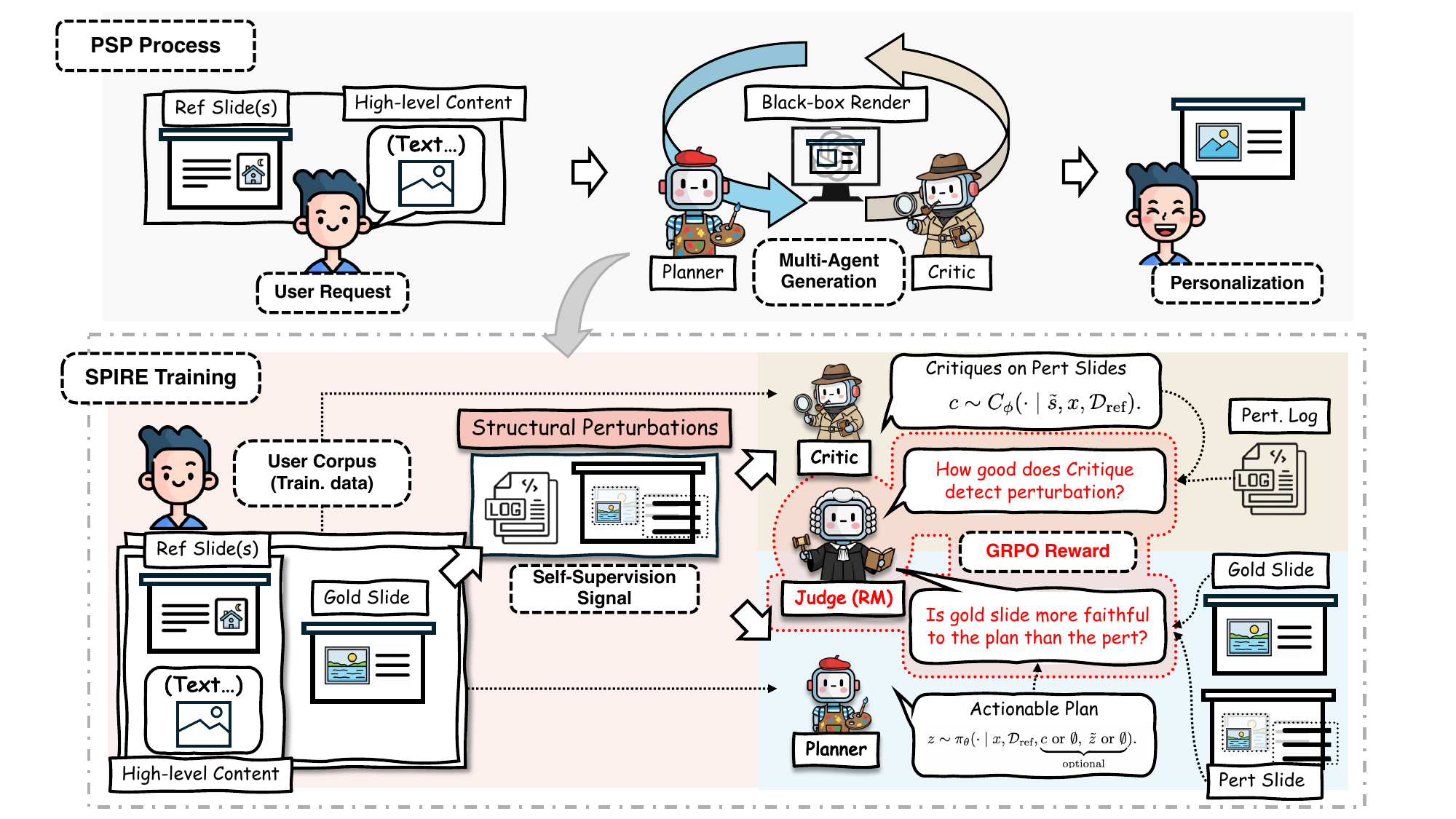}
\caption{
{
{The overview of Personalized Slide Personalization (PSP) and the proposed {\name}.}
The Planner and Critic are trained to recover self-supervised Structural Perturbation in a decomposed way with reinforcement learning (red).
After training, they collaborate with a black-box render executor to conduct PSP.
}
}
\label{fig:overview}
\end{figure*}

To tackle the computational challenge of Eq.~\eqref{eq:formal}, we approximate the two parts therein with two complementary agents based on their intrinsic goals.
The two agents are trained to exploit the discrete structural nature of slides via a shared structural perturbation process, turning the problem into a verifiable reconstruction signal through a
\emph{denoising} process over the slide's discrete structural space.
We dub our method
{\underline{S}tructural \underline{P}lanning via \underline{I}nverse \underline{RE}construction} (\name).
Fig~\ref{fig:overview} depicts the overview of PSP and {\name}.

We begin by checking the roles of Eq.~\eqref{eq:challenge}
{
\footnotesize
\begin{align}
\label{eq:challenge}
p_\theta(s^* \mid x, \mathcal{D}_{\text{ref}})
=
\int
\underbrace{\pi_\theta(z \mid x, \mathcal{D}_{\text{ref}})}_{
\begin{array}{c}
\text{planner }
\text{proposal,}
\end{array}}
\underbrace{p(s^* \mid z)}_{
\begin{array}{c}
\text{render }
\text{likelihood}
\end{array}}
\, dz ,
\end{align}
}%
where the integrand consists of two parts.
The first term, $\pi_\theta(z \mid x, \mathcal{D}_{\text{ref}})$, denotes how the planner $\pi_\theta$ proposes a plan $z$ to reflect the user's latent intent.
Next, the second term $p(s^* \mid z)$ measures the \emph{render} likelihood of the golden slide $s^*$, depicting the plan's \emph{visual} validity.
PSP wants to adjust the planner proposal based on $p(s^* \mid z)$, which, unfortunately, is intractable due to black-box $E$.

To bypass this challenge,
we approximate the update by incorporating a \emph{critic} agent $C_\phi$ to act as a learned proxy for the likelihood.
By learning to \emph{discriminate} how a generation $s$ visually deviates from the gold $s^*$,
the critic captures the user's implicit preferences and provides actionable semantic feedback,
thereby providing guidance for the planner to update its proposal $\pi_\theta(z \mid x, \mathcal{D}_{\text{ref}})$ in context $\mathcal{D}_{\text{ref}}$.
With a reliable critic, the planner can iteratively improve its proposal: starting from an initial plan, it revises the plan based on the critique.

Yet, reliably learning such a critic, and thus enabling critique-guided planner updates, is also non-trivial.
The difficulty is the lack of \emph{verifiable} supervision: for a random generation $s$, one can hardly evaluate whether a critique is objectively correct, which makes subsequent planner training unreliable as well.

\textbf{Denoising Signal.}
To solve this problem, we propose a ``{denoising}'' objective through a corruption-reconstruction strategy.
Namely, we construct a verifiable self-supervised signal that can be shared across the two agents' training.
Specifically, we corrupt the gold slide $s^*$ into a perturbed slide $\tilde{s}$ in a controlled way, and record the ground-truth corruption operations.
Next, the critic is asked to criticize $\tilde{s}$ based on visual evidence, and its critique can be verified by checking the recorded corruption operations.
The same corruption record also helps train the planner to perform critique-guided recovery in the \emph{plan space}: it learns to revise a suboptimal plan $\tilde z$ using the critique $c$.
We will detail this soon.

\textbf{Structural Noises.}
We leverage the discrete structural nature of slides to design \emph{structural corruptions}, rather than adding pixel-level noise to the visuals.
Specifically, we represent each slide $s^*$ as a collection of visual elements that incorporate different element types.
Then, we apply random perturbation to each visual element by modifying its structural attributes.
Each element type admits its own perturbation spaces, e.g., position for layout, size for hierarchy, and color palette for styling.
The corruption occurs along three dimensions: \emph{spatial layout} (e.g., shifting coordinates), \emph{visual hierarchy} (e.g., distorting element sizes), and \emph{stylistic coherence} (e.g., altering text colors).
To prevent spurious correlations, we independently perturb each randomly selected element.

Formally, assume $s^*$ contains $N$ visual elements $\{e_i \}_{i=1}^{N}$, and let $\mathcal{T}(e_i)$ denote the element type (e.g., text box, image, shape).
We represent the corruption by an element-level action list $\mathcal{A}=\{a_i\}_{i=1}^{N}$, where each $a_i$ perturbs some attributes of $e_i$ or leaves it unchanged.
Perturbation actions are sampled independently, i.e.,
{\small
\begin{align*}
q(\mathcal{A} \mid s^*)
&= q(a_i, \dots, a_N \mid s^*)
= q(a_i, \dots, a_N \mid e_i, \dots, e_N)
\overset{(a)}{=} \prod_{i=1}^{N} q_{\mathcal{T}(e_i)}(a_i \mid e_i),
\end{align*}
}%
where $(a)$ holds by independence.
Here $q_{\mathcal{T}(e_i)}(\cdot \mid e_i)$ denotes a perturbation distribution over actions for element $e_i$ (admitting no perturbation).
Subscript $\mathcal{T}(e_i)$ indicates that action spaces depend on the element types.
Having the sampled action list $\mathcal{A}$, we apply these actions to $s^*$ and denote the corrupted slide by $\tilde{s}$.
The inverse of each applied perturbation forms a discrepancy list:
\begin{align*}
\mathcal{A}_{\text{diff}} = \{ a_i^{-1} \mid a_i \neq \emptyset, 1 \leq i \leq N \},
\end{align*}
where $\emptyset$ denotes leaving the element unchanged.
This yields self-supervised triplets $(\tilde{s}, s^*, \mathcal{A}_{\text{diff}})$, providing scalable supervision without manual annotation.
Built upon the structural perturbation, \name~factorizes the intractable PSP problem from Eq.~\eqref{eq:formal} into the following two complementary sub-tasks.

\textbf{Structural Discrimination.}
This task aims to train critic $C_\phi$ to identify and articulate the discrepancies in $\mathcal{A}_{\text{diff}}$,
which correspond to moving $\tilde{s}$ toward the gold $s^*$
to maximize the \emph{render likelihood} term in Eq.~\eqref{eq:challenge}.
Specifically,
the critic generates a critique $c$ based on $x$, a suboptimal visual $\tilde{s}$, and reference $\mathcal{D}_{\text{ref}}$:
\begin{align*}
c \sim C_\phi(\cdot \mid \tilde{s}, x, \mathcal{D}_{\text{ref}}).
\end{align*}
Here, $c$ encompasses a list of issues and targeted corrections, translating errors in $\tilde{s}$ into actionable textual feedback.
We evaluate $c$ against the full action list $\mathcal{A}$ with a capable LLM judge (GPT-4o-mini).
For each element $e_i$ ($1 \leq i \leq N$), the judge verifies whether $c$ correctly handles that element:
\begin{align*}
y_i = J(c, e_i, a_i) \in \{0, 1\},
\qquad
1 \leq i \leq N.
\end{align*}
Here $y_i=1$ indicates a correct decision: the critique identifies and corrects the perturbation when $a_i \neq \emptyset$, or correctly refrains from reporting an issue when $a_i = \emptyset$; and 0 otherwise.
Subsequently, we define the element-level verification reward as the average verification accuracy over all $N$ elements:
\begin{align*}
R_{\text{vfy}}(c, \mathcal{A})
=
\frac{1}{N}\sum_{i=1}^{N} y_i.
\end{align*}
We train the critic using DAPO~\cite{yu2025dapo}, a strong reinforcement learning algorithm, to maximize the expected verification-format reward:
\begin{align}
\label{eq:critic_obj}
\mathcal{J}_{\text{critic}}(\phi)
=
\mathbb{E}_{\substack{s^* \sim \mathcal{D}_{\text{tr}}\\ \tilde{s} \sim q(\cdot \mid s^*)}}
\Big[
\mathbb{E}_{c \sim C_\phi(\cdot \mid \tilde{s}, x, \mathcal{D}_{\text{ref}})}
\big[
R_{\text{vfy}}(c, \mathcal{A}) \times R_\text{format,c}(c)
\big]
\Big].
\end{align}
We detail critique generation and format requirement in the supplementary material.
Note that our perturbation design, combined with element-level evaluation, mitigates reward hacking.
As each perturbation type is applied independently with probability $1/2$, 
trivial always-issue or never-issue critiques cannot consistently score high.

\textbf{Structural Planning.}
This task aims to produce actionable plans that translate user instructions into high-quality slides.
The goal is to approximately improve the \emph{planner proposal} term in Eq.~\eqref{eq:challenge}.
To this end, we train the planner to (i) generate a good plan proposal, and (ii) update a proposal using the critic's critique as a semantic gradient.
Formally, we refer to this goal of the planner as \emph{iterative refinement}, which covers \emph{initial proposal} (generating $z$ from scratch) and \emph{critique-guided refinement} (revising a suboptimal $\tilde z$ based on critique $c$).
The two subgoals can be expressed as a unified formulation.
Given instruction $x$, the reference $\mathcal{D}_{\text{ref}}$, and optionally a suboptimal plan $\tilde z$ to revise, with its corresponding critique $c$,
the planner aims to generate a high quality plan $z$ as
\begin{align}
\label{eq:plan}
z \sim \pi_\theta(\cdot \mid x, \mathcal{D}_{\text{ref}}, \underbrace{c \text{ or } \emptyset,\; \tilde z \text{ or } \emptyset}_{\text{optional}}).
\end{align}
The visual quality of a plan $z$ is measured with
a discriminative \emph{plan--visual matching reward}.
Specifically, a capable VLM judge (Claude Opus 4.5) is given the plan $z$ and a pair of visuals $(s^*, \tilde{s})$, and is asked to choose which one matches the
plan $z$ better. We reward $z$ if the judge prefers the gold slide $s^*$.
The judgment is conducted twice with the presenting order of the two visuals swapped to mitigate the positional bias~\cite{xu2026alternating}.
Given the presenting order $(s^*, \tilde s)$, we denote the judgment outcome as 
\begin{align*}
\hat o_\rightarrow = {\ones} \Big[ s^* \succ \tilde{s} \,\big|\, z, (s^*, \tilde s) \Big]  \in \{0, 1\},
\end{align*}
and $\hat o_\leftarrow$ is defined similarly for the swapped presenting order. Based on this,
we define the
\emph{plan--visual matching reward} as
\begin{align*}
R_{\text{pvm}}(z, s^*, \tilde{s})
&=
\frac{1}{2}\Big(
{\ones}[ \hat o_\rightarrow =1 ]
+
{\ones}[ \hat o_\leftarrow = 1]
\Big)
\end{align*}
To make $R_{\text{pvm}}$ reliable, we instruct the capable judge to base its decision \emph{solely} on fidelity to $z$ rather than its own aesthetic preference.
Note that $\tilde{s}$ is produced through structural perturbations that alter \emph{design styles} (e.g., color palette) instead of making the slide visually worse.
Without seeing $x$ and $\mathcal{D}_{\text{ref}}$, $\tilde{s}$ may look good on its own; therefore, correct \emph{judgment} requires a discriminative $z$.
As with the critic, we train the planner with DAPO to maximize the expected reward, combined with a format verification term:
{
\small
\begin{align*}
\mathcal{J}_{\text{plnr}}(\theta)
&= \mathbb{E}_{\substack{s^* \sim \mathcal{D}_{\text{tr}}\\ \tilde{s} \sim q(\cdot \mid s^*)}}
\Big[
\mathbb{E}_{c^* \sim \text{Oracle}(c)}
\big[
\mathbb{E}_{z \sim \pi_\theta(\cdot \mid x, \mathcal{D}_{\text{ref}}, c^*)}
[
R_{\text{pvm}}(z, s^*, \tilde{s}) \times R_\text{format,p}(z)
]
\big]
\Big].
\end{align*}
}%
We defer more details about plan generation and format requirements \textcolor{purple} to the supplementary material.
Note that the iterative refinement goal designed for the planner requires high-quality $(c^*, \tilde z)$ to train the agent so that it can reliably refine a given plan under critique.
To this end, we use an oracle VLM (GPT-5) to (i) describe the perturbed slide $\tilde{s}$ as a suboptimal plan $\tilde z$, and (ii) translate the discrepancy list $\mathcal{A}_{\text{diff}}$ (provided as \emph{hints}~\cite{zelikman2022star,zelikman2024quiet,li2025start}) into gold critique $c^*$:
\begin{align*}
\tilde z \sim \text{Oracle}(\tilde z \mid \tilde{s}),
\qquad
c^* \sim \text{Oracle}(c \mid \mathcal{A}_{\text{diff}}).
\end{align*}

\subsection{Theoretical Analysis}
\label{sec:theory}
We end up this section with the theoretical advantages of {\name}.
Full versions are deferred to the supplementary material due to page limit.

First, under regular conditions, optimizing the {\name} objective gives a gradient-level
approximation to that of the original PSP objective up to an explicit error bound, as formalized in the following Thm~\ref{thm:surrogate_informal}.

\begin{theorem}[Surrogate Validity for PSP, informal]
\label{thm:surrogate_informal}
Let $\widehat{\mathcal J}_{\text{SP}}(\theta,\phi)$ be the empirical loss of \name. Then there exists $\epsilon_{\mathrm{tot}}\ge 0$ such that
{
\small
\begin{align*}
\left\|
\nabla_\theta \widehat{\mathcal J}_{\text{SP}}(\theta,\phi)
-\frac{\beta}{4}\nabla_\theta \mathcal{J}(\theta)
\right\|
\le
\epsilon_{\mathrm{tot}}.
\end{align*}
}%
Here $\epsilon_{\mathrm{tot}}$ aggregates critic-calibration, structural-surrogate, variational-gap, and empirical optimization errors. Full definitions are provided in the supplementary material.
\end{theorem}

Crucially, the key error term is controlled thanks to our critic training on structural corruptions with verifiable ground truth supervision. An untrained or
pixel-denoising-trained C would not satisfy this bound in general.

Furthermore, Thm~\ref{thm:decomp_informal} below proves that \name trains the planner and critic \emph{separately}, thereby eliminating executor-induced noise.

\begin{theorem}[Two-agent Decomposition Stabilizes Optimization, informal]
\label{thm:decomp_informal}
Let $\hat g_{\text{e2e}}$ and $\hat g_{\text{2a}}$ denote the end-to-end and two-agent policy-gradient estimators, respectively; see the supplementary material for explicit definitions and derivation.
Then
{\small
\begin{align*}
\mathrm{Var}(\hat g_{\text{e2e}})
=
\mathrm{Var}(\hat g_{\text{2a}})
+
\Delta_{\text{exec}\mid c},
\quad
\Delta_{\text{exec}\mid c}\ge 0.
\end{align*}
}%
So $\mathrm{Var}(\hat g_{\text{2a}})\le \mathrm{Var}(\hat g_{\text{e2e}})$, with strict inequality for non-zero executor-induced noise.
See formal statement, expression, and imperfect-critic extension in the supplementary material.
\end{theorem}

In general, even if a planner-executor-critic design is adopted, standard training of the planner and critic still requires the executor to generate visuals for complete rollouts, retaining its induced noise.
{\name} provides a way to bypass the executor in our decomposed training, thereby enabling the noise reduction.

\section{Experiment}
\label{sec:exp}

We evaluate the performance of {\name} and ablate its components' contributions. 
Benefiting from the principled optimization presented before, {\name} offers strong PSP capability over strong baselines that rely on much larger GPT models.

\subsection{Dataset and Experiment Settings}

\textbf{Datasets.}
We build the \textbf{train} and \textbf{test} data from Zenodo10k~\cite{zheng2025pptagent} by randomly sampling 200 decks. For each deck, we perform a slide-level held-out split that preserves the deck’s chronological order. 
The last 20\% of slides are reserved for testing and the remaining slides are used for training (or for retrieval for training-free baselines).
SlideBench~\cite{ge2025autopresent} is used as \textbf{out-of-distribution} test data. 
We defer more data preparation details in the supplementary material.

\textbf{Backbones.}
Both the critic and the planner are fine-tuned from Qwen2.5-VL-7B-Instruct~\cite{bai2025qwen25vl}.
At inference time, the planner and critic collaborate with a black-box executor $E$ to perform multi-round iterative generation.
We use GPT-o4-mini as the coding executor to implement the plan with \texttt{python-pptx}, and follow \cite{ge2025autopresent,tang2025slidecoder} to improve coding reliability. No template is used in execution.

\textbf{Baselines.}
We compare \name\ against representative systems that cover both slide-native pipelines and generic visual synthesis.
We include {AutoPresent}~\cite{ge2025autopresent}, which generates a slide by deciding layout and style for the given instruction from scratch on its own.
We also include {PPTAgent}~\cite{zheng2025pptagent}, which selects a template/style from the reference slides and re-applies it to the target content through an edit-based workflow.
Finally, following \cite{ge2025autopresent}, we include a generic image generation baseline using {Stable Diffusion 3.5}~\cite{esser2024scaling} equipped with IP-Adapter~\cite{ye2023ip} which prompt the model to synthesize a slide-like image.
To isolate the benefit of our training, we additionally report PSP results with planner/critic setting to: (i) frontier VLM (o4-mini) and (ii) the corresponding base model (without fine-tuning), while keeping the rest of the pipeline unchanged.

\textbf{Evaluations.}
We evaluate the generation in terms of visual similarity against the gold slides, and reference-free quality.
Following \cite{ge2025autopresent,tang2025slidecoder,liu2025presenting}, we adopt both visual-similarity metrics and a VLM-as-a-Judge protocol. 
Specifically, for visual similarity, we treat the slides as rendered images and utilize metrics such as SSIM and CLIP to measure ``reconstruction'' quality.  
For VLM-as-a-judge evaluation, we follow \cite{ge2025autopresent,zhu2026paperbanana,liu2025presenting} and evaluate the generation quality from the perspectives of \emph{faithfulness}, \emph{color}, \emph{layout}, and \emph{overall aesthetics}. 
We detail these in the supplementary material.

\subsection{Quantitative Results}
\label{sec:exp:quant}

\begin{table*}[t!]
\caption{
{Quantitative results on \emph{test} and \emph{OOD} pages.}
Best average results are in {\bf bold} and the second best results are \underline{underlined}. 
}
\label{tab:main}

\centering
\resizebox{1.0\linewidth}{!}{%
\renewcommand{\tabcolsep}{3pt}
\begin{tabular}{
r
cc
>{\columncolor{gray!10}}c
cccc
>{\columncolor{gray!10}}c
}
\toprule[0.4ex]
\multirow{2}{*}{\bf Model} & \multicolumn{3}{c}{\bf Visual Similarity} & \multicolumn{5}{c}{\bf VLM-as-Judge} \\
\cmidrule[0.2ex](lr){2-4} \cmidrule[0.2ex](lr){5-9}
& $\mathrm{Sim}_{\text{ssim}} \uparrow$ & $\mathrm{Sim}_{\text{clip}}\uparrow$ & \bf AVG
& Faith $\uparrow$ & Color $\uparrow$ & Layout $\uparrow$ & Aest $\uparrow$ & \bf AVG \\
\midrule[0.2ex]

\multicolumn{9}{c}{\textbf{Test Pages}} \\
\midrule[0.2ex]
& \multicolumn{8}{c}{\emph{GPT-based Models}} \\
\noalign{\vskip 0.2ex}\cdashline{2-9}\noalign{\vskip 0.2ex}
AutoPresent~\cite{ge2025autopresent} & 0.6062 & 0.4431 & 0.5247 & 0.6174 & {0.5850} & 0.4145 & {0.4105} & \underline{0.5069} \\
PPTAgent~\cite{zheng2025pptagent}    & 0.5126 & 0.5491 & 0.5309 & 0.1036 & 0.5670 & {0.4364} & 0.3961 & 0.3758 \\
PSP (o4-mini)                        & {0.8440} & {0.7683} & \textbf{0.8062} & {0.5504} & 0.5763 & 0.4206 & 0.3661 & 0.4784 \\
\noalign{\vskip 0.2ex}\cdashline{2-9}\noalign{\vskip 0.2ex}
& \multicolumn{8}{c}{\emph{7B-level Models}} \\
\noalign{\vskip 0.2ex}\cdashline{2-9}\noalign{\vskip 0.2ex}
SD 3.5~\cite{esser2024scaling}       & 0.3372 & 0.4496 & 0.3934 & 0.1559 & 0.4282 & 0.2545 & 0.1889 & 0.2569 \\
PSP (Base)                           & {0.3578} & 0.3317 & 0.3448 & 0.2995 & 0.4588 & 0.2956 & 0.2382 & 0.3230 \\
\noalign{\vskip 0.2ex}\cdashline{2-9}\noalign{\vskip 0.2ex}
{\name} (Ours)                       & {0.8134} & {0.7018} & \underline{0.7576} & {0.7072} & {0.6330} & {0.4470} & {0.3787} & \textbf{0.5415} \\
\midrule[0.2ex]

\multicolumn{9}{c}{\textbf{OOD Pages}} \\
\midrule[0.2ex]
& \multicolumn{8}{c}{\emph{GPT-based Models}} \\
\noalign{\vskip 0.2ex}\cdashline{2-9}\noalign{\vskip 0.2ex}
AutoPresent~\cite{ge2025autopresent} & 0.5976 & 0.5892 & 0.5934 & 0.7923 & 0.7242 & 0.5477 & 0.5201 & \underline{0.6461} \\
PPTAgent~\cite{zheng2025pptagent}    & 0.4836 & 0.6371 & 0.5604 & 0.0839 & 0.5824 & 0.4392 & 0.3831 & 0.3722 \\
PSP (o4-mini)                        & 0.7233 & 0.7633 & {\bf 0.7433} & 0.6975 & 0.5947 & 0.4472 & 0.3958 & 0.5338 \\

\noalign{\vskip 0.2ex}\cdashline{2-9}\noalign{\vskip 0.2ex}
& \multicolumn{8}{c}{\emph{7B-level Models}} \\
\noalign{\vskip 0.2ex}\cdashline{2-9}\noalign{\vskip 0.2ex}
SD 3.5~\cite{esser2024scaling}       & 0.3177 & 0.5753 & 0.4465 & 0.1811 & 0.5075 & 0.2934 & 0.2246 & 0.3016 \\
PSP (Base)                           & 0.3910 & 0.3864 & 0.3887 & 0.2432 & 0.2309 & 0.1716 & 0.1292 & 0.1937 \\
\noalign{\vskip 0.2ex}\cdashline{2-9}\noalign{\vskip 0.2ex}
{\name} (Ours)                       & 0.6785 & 0.6954 & \underline{0.6870} & 0.9330 & 0.7545 & 0.6565 & 0.5891 & {\bf 0.7333} \\

\bottomrule[0.4ex]
\end{tabular}
}
\end{table*}

Tab~\ref{tab:main} shows quantitative results on test and out-of-distribution (OOD) pages. 
We note the following observations.

\textbf{Good PSP Performance.}
From the table, {\name} achieves the strongest overall performance, outperforming the GPT-based AutoPresent (0.5069) and PSP (o4-mini) (0.4784) in terms of judge score, while maintaining a highly competitive visual similarity score (0.7414), substantially higher than the 7B-level baselines SD~3.5 and PSP (Base). This performance is notable given that {\name} relies on only two 7B-scale agents, yet achieves competitive performance compared to GPT-based components. 
We also note empirical failures in the baselines, attributed to their underlying mechanisms. 
PPTAgent struggles with faithfulness (0.1036), as realistic reference sets rarely offer perfect templates without structural loss\footnote{In reality, \texttt{pptx} template files may not be accessible.}.
AutoPresent falls short when the verbose prompting is too coarse.
\emph{These trends clearly support our PSP formulation.}
The suboptimal results of the training-free PSP baselines highlight the necessity of preference-aligned optimization; without it, the critic leverages its own generic preferences rather than reflecting the user's contextual needs. 
\emph{This confirms that {\name}'s gains stem from the RL training, not merely the multi-agent refinement loop.}

\textbf{Strong Generalizability.}
On OOD pages, {\name} demonstrates strong generalizability and achieves the highest judge average, substantially outperforming baselines such as AutoPresent, PSP (o4-mini), PPTAgent, SD~3.5, and PSP (Base). 
Our method obtains the best scores across all judging dimensions and the second-best visual similarity. 
These results indicate that {\name} does not simply memorize deck-specific patterns from the training corpus, but instead successfully leverages the reference slides as contextual evidence at inference time, allowing it to infer the appropriate design intent in unseen scenarios.

\textbf{Metric Discrepancy.}
We also note that visual similarity and judge-based scores are not perfectly aligned. 
For instance, while PSP (o4-mini) achieves the highest visual averages across both in-distribution and OOD settings, its judge-based quality remains substantially lower than that of {\name}. 
This mismatch, as explained in Sec.~\ref{sec:method}, highlights the challenge of delivering PSP by optimizing standardized metrics alone---a limitation also noted in the literature~\cite{ge2025autopresent,qu2025igd,xu2025pregenie}.

These results collectively demonstrate the effectiveness of the proposed {\name}.

\ExplSyntaxOn
\cs_set:Npn \loadimage #1#2#3#4 {
    \seq_clear:N \l_tmpa_seq
    \int_step_inline:nnn {#3} {#3 + #4 - 1} {
        \file_if_exist:nTF {#1/##1.pdf} 
        {
            \seq_put_right:Nx \l_tmpa_seq {
                \exp_not:N \adjustbox {valign=m} {
                    \exp_not:N \parbox [c] [\exp_not:N \figheight] [c] {\exp_not:N \figwidth} {
                        \exp_not:N \centering
                        \exp_not:N \includegraphics [max~width=\exp_not:N \figwidth, max~height=\exp_not:N \figheight, keepaspectratio, \exp_not:n {#2}] {#1/##1.pdf}
                    }
                }
            }
        }
        {
            \seq_put_right:Nx \l_tmpa_seq {
                \exp_not:N \adjustbox {valign=m} {
                    \exp_not:N \parbox [c] [\exp_not:N \figheight] [c] {\exp_not:N \figwidth} {
                        \exp_not:N \centering
                        \exp_not:N \includegraphics [max~width=\exp_not:N \figwidth, max~height=\exp_not:N \figheight, keepaspectratio, \exp_not:n {#2}] {example-image-a}
                    }
                }
            }
        }
    }
    \seq_use:Nn \l_tmpa_seq {&}
}
\ExplSyntaxOff

\begin{figure*}[tb!] 
\centering
\renewcommand{\tabcolsep}{1.3pt}
\def\figwidth{0.11\textwidth} 
\def\figheight{1.2cm} 

\newcommand{\rowgroup}[1]{
    \adjustbox{rotate=90}{{#1}}
}

\begin{tabular}{c c ccccc cc} 
\toprule[0.4ex]

& \scriptsize Ref 
& \scriptsize AutoPre. 
& \scriptsize PPTAgent 
& \scriptsize SD 3.5 
& {\scriptsize \begin{tabular}[c]{@{}c@{}}PSP\\(o4-mini)\end{tabular}} 
& {\scriptsize \begin{tabular}[c]{@{}c@{}}PSP\\(Base)\end{tabular}} 
& \scriptsize Ours 
& \scriptsize Gold \\
\noalign{\vskip 0.5ex}
\midrule[0.2ex]

\multirow{2}{*}[-0.85cm]{\rowgroup{Test}} 
& \loadimage{figures/12_s_122_53}{width=\figwidth, height=\figheight, keepaspectratio}{0}{8} \\
& \loadimage{figures/63_s_115_8}{width=\figwidth, height=\figheight, keepaspectratio}{0}{8} \\
& \loadimage{figures/93_s_152_9}{width=\figwidth, height=\figheight, keepaspectratio}{0}{8} \\

\noalign{\vskip 0.2ex}\cdashline{2-9}\noalign{\vskip 0.2ex}

\multirow{2}{*}[-0.45cm]{\rowgroup{OOD}} 
& \loadimage{figures/social_media_21}{width=\figwidth, height=\figheight, keepaspectratio}{0}{8} \\
& \loadimage{figures/business_2}{width=\figwidth, height=\figheight, keepaspectratio}{0}{8} \\

\bottomrule[0.4ex]
\end{tabular}
\caption{
Qualitative comparison of slide generation results.
}
\label{fig:main}
\end{figure*}

\subsection{Qualitative Results}
\label{sec:exp:qual}

We provide a visual comparison of the generated slides across different methods.
Illustrative results are shown in Fig~\ref{fig:main}, with more examples deferred to the supplementary material. 
We note that {\name} consistently produces visually appealing and structurally coherent slides that closely align with the gold slides, 
while baselines often struggle to maintain visual fidelity or fail to generate valid structures.

The effectiveness of {\name}, stemming from our preference-aligned RL training, empowers the model to explicitly infer latent user intent. This capability manifests in two aspects. 
First, {\name} achieves superior color coherence. 
For instance, as shown in Row 4 and 5, {\name} accurately \emph{infers} a proper coloring strategy that differs from the reference slide in OOD scenarios. 
Second, {\name} demonstrates strong spatial reasoning for layout arrangement.
In Row 2, it successfully organizes complex, multi-image visual assets (e.g., heatmaps) into a clean, aligned structure that matches the gold slide perfectly. In Row 3, it also correctly positions the logo and text blocks.
This proactive inference capability allows {\name} to surpass baselines solely relying on verbose instructions (AutoPresent) or generic templates (PPTAgent). When such explicit inputs are unavailable, these methods inevitably fail to capture the user's true intent.

\subsection{Ablation Study}
\label{sec:exp:abla}

\begin{table*}[t!]
\caption{
{Ablation study on {\name} Training and Iterative Revision}. Both components are necessary for high-quality generation.
}
\label{tab:abla:spire}
\centering
\resizebox{1.0\linewidth}{!}{%
\renewcommand{\tabcolsep}{4pt}
\begin{tabular}{
r
cc
>{\columncolor{gray!10}}c
cccc
>{\columncolor{gray!10}}c
}
\toprule[0.4ex]
\multirow{2}{*}{\bf Model} & \multicolumn{3}{c}{\bf Visual Similarity} & \multicolumn{5}{c}{\bf VLM-as-Judge} \\
\cmidrule[0.2ex](lr){2-4} \cmidrule[0.2ex](lr){5-9}
& $\mathrm{Sim}_{\text{ssim}} \uparrow$ & $\mathrm{Sim}_{\text{clip}}\uparrow$ & \bf AVG
& Faith $\uparrow$ & Color $\uparrow$ & Layout $\uparrow$ & Aest $\uparrow$ & \bf AVG \\
\midrule[0.2ex]
w/o FT  & 0.3578 & 0.3317 & 0.3448 & 0.2995 & 0.4588 & 0.2956 & 0.2382 & 0.3230 \\
w/o Revision & 0.7299 & 0.6096 & 0.6698 & 0.4865 & 0.5120 & 0.3386 & 0.2779 & 0.4038 \\
\noalign{\vskip 0.2ex}\cdashline{2-9}\noalign{\vskip 0.2ex}
{\name}  & {0.8134} & {0.7018} & {0.7576} & {0.7072} & {0.6330} & {0.4470} & {0.3787} & {0.5415} \\
\bottomrule[0.4ex]
\end{tabular}
}
\end{table*}

\begin{table*}[t!]
\caption{
{Ablation study on the critic role.} 
Substituting the larger, untrained GPT critic with {\name}-trained 7B model results in improved performance.
}
\label{tab:abla:critic}

\centering
\resizebox{1.0\linewidth}{!}{%
\renewcommand{\tabcolsep}{4pt}
\begin{tabular}{
r
cc
>{\columncolor{gray!10}}c
cccc
>{\columncolor{gray!10}}c
}
\toprule[0.4ex]
\multirow{2}{*}{\bf Model} & \multicolumn{3}{c}{\bf Visual Similarity} & \multicolumn{5}{c}{\bf VLM-as-Judge} \\
\cmidrule[0.2ex](lr){2-4} \cmidrule[0.2ex](lr){5-9}
& $\mathrm{Sim}_{\text{ssim}} \uparrow$ & $\mathrm{Sim}_{\text{clip}}\uparrow$ & \bf AVG
& Faith $\uparrow$ & Color $\uparrow$ & Layout $\uparrow$ & Aest $\uparrow$ & \bf AVG \\
\midrule[0.2ex]
w/ o4-mini Critic    & 0.8440 & 0.7683 & 0.8062 & 0.5504 & 0.5763 & 0.4206 & 0.3661 & 0.4784 \\
w/ {\name} Critic    & 0.9367 & 0.8633 & 0.9000 & 0.5777 & 0.5905 & 0.4515 & 0.3942 & 0.5035 \\
\bottomrule[0.4ex]
\end{tabular}
}
\end{table*}

\begin{figure}[t] 
\centering
\begin{subfigure}[b]{0.49\textwidth}
    \centering
    \includegraphics[width=\textwidth]{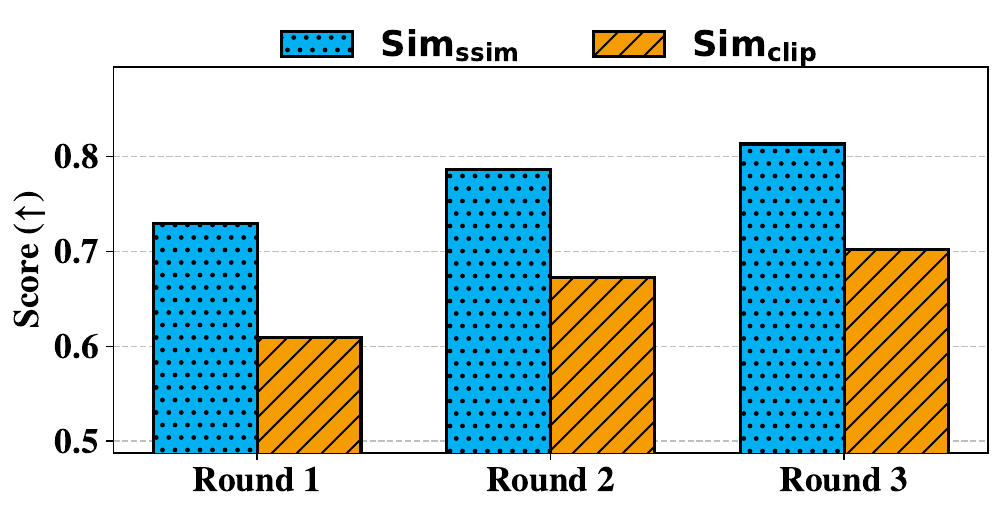}
    \caption{Visual similarity.}
    \label{fig:abla:vis}
\end{subfigure}
\begin{subfigure}[b]{0.49\textwidth}
    \centering
    \includegraphics[width=\textwidth]{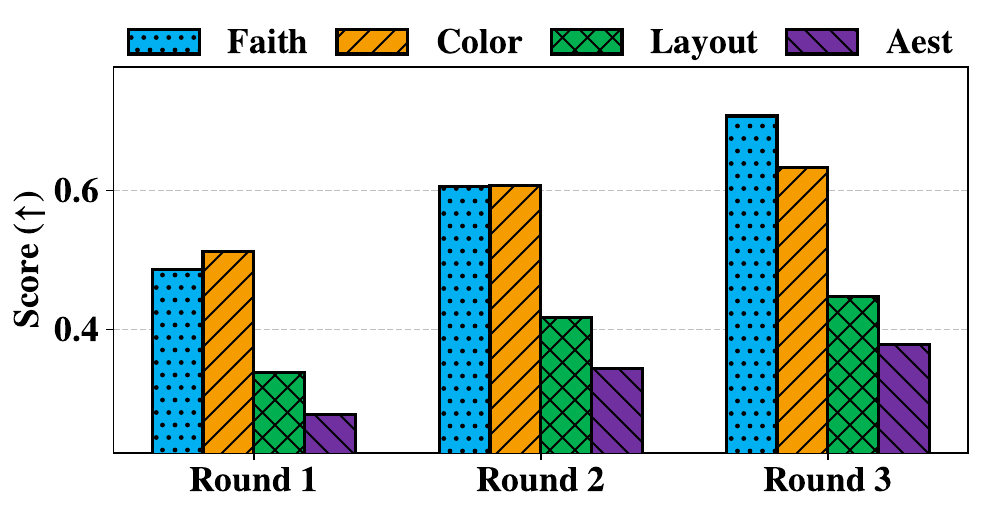}
    \caption{Judge scores.}
    \label{fig:abla:vlm}
\end{subfigure}
\caption{
Ablation on {\name} revision.
Multi-round revision in {\name} helps improve both visual similarity and judge scores, indicating better generation qualities. 
}
\label{fig:abla:revision}
\end{figure}

We end up this section with ablation study on key components in {\name}.

\textbf{{\name} Training and Iterative Revision.}
To break down the contributions of our proposed framework, we ablate the fine-tuning process and the multi-turn revision mechanism, with results reported in Tab.~\ref{tab:abla:spire}.
First, we observe that preference-aligned fine-tuning is necessary for PSP. Namely, when directly prompting base Qwen models to conduct PSP, they struggle significantly and achieve a visual average of only 0.3344 and a judge average of 0.3230. In contrast, our fine-tuned {\name} achieves much better scores, reaching 0.7414 and 0.5415 respectively. This massive performance gap confirms that off-the-shelf VLMs, especially small-scale models, cannot reliably infer latent page-level design intents without our targeted structural denoising optimization.
Second, the results validate the effectiveness of the iterative revision process. 
As shown in the table, and Fig.~\ref{fig:abla:revision},
the visual similarity scores exhibit a clear and steady improvement through the successive refinement rounds. This steady climb is a direct result of our structural denoising objective, which explicitly trains the planner to interpret and execute structural corrections based on the critic's feedback. Consequently, the planner is capable of progressively refining suboptimal layouts, proving that {\name} successfully leverages high-quality critiques to iteratively enhance the design rather than relying on a single-pass generation.

\textbf{Critic Role.}
One hypothesis stated in Sec.~\ref{sec:exp:quant} is that the suboptimal performance of PSP lies in its critic. 
To verify that the preference-aligned training helps the critic learn the user's contextual needs better, we replace the critic in the PSP (o4-mini) baseline with our 7B-level {\name}-trained critic and report the results on test slides.
Results are reported in Tab~\ref{tab:abla:critic}. From the table, we see that this substitution yields consistent improvements across all metrics. Remarkably, the critic is a 7B-level model, which is much less capable than GPT models. This gain confirms that such a targeted critic indeed helps guide the PSP process.

\section{Related Work}
\label{sec:relate}

Early works consider automated slide generation as a text summarization task~\cite{sun2021d2s,fu2022doc2ppt,cao2025multi2},
aiming to extract~\cite{sun2021d2s,fu2022doc2ppt} and summarize~\cite{cao2025multi2} proper deck-level content from a given document for presentation. These solutions boiled down to extracting salient sentences, figures, and sections from source documents to organize them into slide outlines. 
However, they overlooked the necessity of coherent layout design and the inherent multimodal nature of slides~\cite{zeng2025slidetailor,liang2025slidegen}. 
As a result, these solutions offered limited control over the aesthetic perspective of the generated slides, leaving personalized generation untouched.

Following the emergence of (M)LLMs, recent research has shifted toward end-to-end and agentic pipelines for slide design~\cite{mondal2024presentations,ge2025autopresent, zheng2025pptagent,zeng2025slidetailor,pang2025paper,xie2025slidebot,jung2025talk,jang2026deckbench}.
Representative solutions decompose the task into modular stages, including content outlining, asset extraction, layout arrangement, and iterative refinement. 
These systems improve visual consistency through template-based selection, visual feedback, and multi-agent coordination~\cite{mondal2024presentations,zheng2025pptagent,zeng2025slidetailor,pang2025paper,xie2025slidebot,jung2025talk}.
Nevertheless, these solutions mainly focus on more global deck- or document-level decisions and planning, lacking more fine-grained page-level layout design capabilities. 
For such page-level design, these solutions are either guided by generic templates and system-defined objectives~\cite{zheng2025pptagent,zeng2025slidetailor}, or by explicit and lengthy instructions specified by users~\cite{ge2025autopresent,pang2025paper}.
These designs inherently limit their performance for PSP, which requires deep understanding of user-specific design intents.

Very recent works have explored personalization for slide generation. However, existing works mostly focus on adapting content and narrative structure to different audiences or presentation contexts~\cite{jung2025talk,zheng2025pptagent,zeng2025slidetailor}.
\cite{mondal2024presentations} enables more tailored communication by conditioning on audience profiles or example pairs,
and \cite{zeng2025slidetailor} allows users to specify a document--slide deck pair to express their preference, with the system aiming to replicate reference slides while plugging in content from the user's own document. 
However, these designs still lack the ability to create novel page-level designs tailored to the user's intent. 
In this work, we propose {\name} to learn a user's page-level preferences from their visual data, pushing agentic slide generation to a more fine-grained page-level design task.
\section{Conclusion}
\label{sec:conclusion}

This paper studies Page-level Slide Personalization (PSP), a key aspect of agentic slide generation that remains largely underexplored. Broadly, we show that personalized visual generation is fundamentally a latent-intent inference problem. We formulate PSP as an inverse planning problem to infer a user's latent design intent, which can be used to guide diverse executors to render the visuals. We propose {\name}, which leverages structural denoising to construct a tractable surrogate objective for optimization. By intentionally corrupting the visual structures of gold slides, {\name} creates a self-supervised denoising task whereby two agents are trained to iteratively refine the design plan. We provide theoretical analysis on how structural denoising serves as a consistent surrogate objective for PSP, and how our multi-agent formulation strictly reduces policy gradient variance to stabilize the optimization. Extensive experiments demonstrate the effectiveness of our method against representative baselines for the PSP task.

\section*{Acknowledgements}
This work is supported in part by the US National Science Foundation under grant NSF IIS-2141037. Any opinions, findings, and conclusions or recommendations expressed in this material are those of the author(s) and do not necessarily reflect the views of the National Science Foundation.

\bibliographystyle{splncs04}
\bibliography{reference}

\appendix
\clearpage
\title{Supplementary Material of \\ Personalization as Inverse Planning: Learning Latent Design Intents for Agentic Slide Generation via Structural Denoising}
\author{
Tianci Liu\inst{1,*} \and
Zihan Dong\inst{2} \and
Linjun Zhang\inst{2} \and
Haoyu Wang\inst{3} \and
Jing Gao\inst{1} \and
Emre K\i c\i man\inst{4} \and
Ranveer Chandra\inst{4} \and
Wei-Ting Chen\inst{4}
}

\authorrunning{T.~Liu et al.}

\institute{
Purdue University \and 
Rutgers University \and 
University at Albany \and
Microsoft
}
\maketitle
\def\thefootnote{*}\footnotetext{Work done during an internship at Microsoft.}\def\thefootnote{\arabic{footnote}}

\section{Theoretical Analysis Details}
\label{app:theory}

This section provides formal statements and proofs for the two informal theorems in the main text.
We only keep the theorem-specific assumptions where each theorem is stated.

\subsection{Preliminaries and Notation}
\label{app:theory:notation}

Each training instance is $(x,s^*,\mathcal D_{\text{ref}})\sim\mathcal D_{\text{tr}}$.
The planner samples a latent plan $z\sim\pi_\theta(\cdot\mid x,\mathcal D_{\text{ref}})$.
Unless explicitly stated otherwise, $\|\cdot\|_2$ denotes $\ell_2$ norm on vectors.
For clarity, we use three objectives throughout this section:
\begin{itemize}
\item $\mathcal J(\theta)$: the original PSP objective in Equation~\eqref{eq:formal}:
\[
\mathcal J(\theta)
=
\mathbb{E}_{(x,s^*,\mathcal D_{\text{ref}})\sim \mathcal D_{\text{tr}}}
\left[
\log \int \pi_\theta(z\mid x,\mathcal D_{\text{ref}})\, p(s^*\mid z)\, dz
\right].
\]
\item $\mathcal J_{\text{LB}}(\theta)$: the Jensen lower bound of $\mathcal J(\theta)$ (i.e., $\mathcal J(\theta)\ge \mathcal J_{\text{LB}}(\theta)$):
\begin{equation}
\mathcal J_{\text{LB}}(\theta):=
\mathbb{E}_{(x,s^*,\mathcal D_{\text{ref}})}
\mathbb{E}_{z\sim\pi_\theta(\cdot\mid x,\mathcal D_{\text{ref}})}
\left[\log p(s^*\mid z)\right].
\label{eq:jlb}
\end{equation}
\item $\mathcal J_{\text{SP}}(\theta,\phi)$: the structural-perturbation surrogate optimized in training:
\begin{equation}
\mathcal J_{\text{SP}}(\theta,\phi):=
\mathbb{E}_{(x,s^*,\mathcal D_{\text{ref}})}
\mathbb{E}_{\tilde s\sim q_\rho(\cdot\mid s^*)}
\mathbb{E}_{z\sim\pi_\theta(\cdot\mid x,\mathcal D_{\text{ref}})}
\left[\widehat R_\phi(z,\tilde s)\right],
\label{eq:jsp}
\end{equation}
where $\widehat R_\phi$ is the learned critic-induced reward.
\end{itemize}

For structural perturbation, sample $\tilde s\sim q_\rho(\cdot\mid s^*)$.
Define
\begin{equation}
m(z,\tilde s):=\log p(s^*\mid z)-\log p(\tilde s\mid z),
\qquad
R^*(z,\tilde s):=\sigma(\beta\,m(z,\tilde s)).
\label{eq:margin_rstar}
\end{equation}
Here $\rho$ controls perturbation strength, $\beta>0$, and $\sigma(\cdot)$ is sigmoid.

In practice, we optimize its empirical estimator
\[
\widehat{\mathcal J}_{\text{SP}}(\theta,\phi),
\]
which is the empirical loss implemented by the critic and planner objectives in the main text.

For variance analysis with critique conditioning, define
\begin{equation}
H(z,c;x):=\nabla_\theta\log\pi_\theta(z\mid x,\mathcal D_{\text{ref}},c).
\label{eq:score_def}
\end{equation}

\subsection{Details of Informal Theorem 1 (Surrogate Consistency)}
\label{app:theory:thm1_details}
\begin{assumption}[A1 (Regularity)]
(a)~$|\beta\,m(z,\tilde s)|\le M$ for sampled $(z,\tilde s)$.
\quad
(b)~$\left\|\,\nabla_\theta\,
\mathbb{E}_{(x,s^*,\mathcal D_{\text{ref}})}
\mathbb{E}_{\tilde s\sim q_\rho}
\mathbb{E}_{z\sim\pi_\theta}
[\log p(\tilde s\mid z)]\,\right\|_2\le \eta$.
\end{assumption}

\begin{assumption}[A2 (Critic-calibration gradient bound)]
\[
\left\|\,\nabla_\theta \mathcal J_{\text{SP}}(\theta,\phi)-\nabla_\theta \mathcal J_{R^*}(\theta)\,\right\|_2\le\delta_{\mathrm{cal}},
\]
where $\mathcal J_{R^*}(\theta):=\mathbb{E}_{(x,s^*,\mathcal D_{\text{ref}})}\mathbb{E}_{\tilde s\sim q_\rho}\mathbb{E}_{z\sim\pi_\theta}[R^*(z,\tilde s)]$.
\end{assumption}

\begin{assumption}[A3 (Gradient gap bounds)]
Under the joint sampling path
\[
\begin{aligned}
(x,s^*,\mathcal D_{\text{ref}})&\sim\mathcal D_{\text{tr}},\qquad
\tilde s\sim q_\rho(\cdot\mid s^*),\\
c&\sim C_\phi(\cdot\mid \tilde s,x,\mathcal D_{\text{ref}}),\qquad
z\sim \pi_\theta(\cdot\mid x,\mathcal D_{\text{ref}},c),
\end{aligned}
\]
the empirical objective $\widehat{\mathcal J}_{\text{SP}}$ is the Monte Carlo estimator of $\mathcal J_{\text{SP}}$, and the following gradient gaps hold:
\[
\left\|\,\nabla_\theta \mathcal J_{\text{LB}}(\theta)-\nabla_\theta \mathcal J(\theta)\,\right\|_2\le\epsilon_{\mathrm{vi}},
\qquad
\left\|\,\nabla_\theta \widehat{\mathcal J}_{\text{SP}}(\theta,\phi)-\nabla_\theta \mathcal J_{\text{SP}}(\theta,\phi)\,\right\|_2\le\epsilon_{\mathrm{emp}}.
\]
\end{assumption}

\begin{theorem}[Surrogate Approximation Bound]
\label{thm:surrogate_consistency}
Under A1--A3,
\begin{equation}
\left\|\,
\nabla_\theta \widehat{\mathcal J}_{\text{SP}}(\theta,\phi)
-\frac{\beta}{4}\nabla_\theta \mathcal J(\theta)
\,\right\|_2
\le
\epsilon_{\mathrm{emp}}
+\delta_{\mathrm{cal}}
+\frac{\beta}{4}\eta
+\beta L_\sigma(M)M\,B_m
+\frac{\beta}{4}\epsilon_{\mathrm{vi}},
\label{eq:surrogate_bound}
\end{equation}
where
\[
B_m:=\mathbb{E}\!\left[\left\|\nabla_\theta m(z,\tilde s)\right\|_2\right],
\qquad
L_\sigma(M):=\sup_{|u|\le M}|\sigma''(u)|.
\]
Consequently, defining the aggregate error
\[
\epsilon_{\mathrm{tot}}
:=
\underbrace{\epsilon_{\mathrm{emp}}}_{\text{empirical opt.}}
+\underbrace{\delta_{\mathrm{cal}}}_{\text{critic-calibration}}
+\underbrace{\tfrac{\beta}{4}\eta+\beta L_\sigma(M)M B_m}_{\text{structural-surrogate}}
+\underbrace{\tfrac{\beta}{4}\epsilon_{\mathrm{vi}}}_{\text{variational-gap}},
\]
we have $\left\|\,\nabla_\theta \widehat{\mathcal J}_{\text{SP}}(\theta,\phi)-\frac{\beta}{4}\nabla_\theta \mathcal J(\theta)\,\right\|_2\le\epsilon_{\mathrm{tot}}$,
which is the formal version of Informal Theorem~1 in the main text.
\end{theorem}

\begin{proof}
We bound the total error in three steps via the intermediate objectives $\mathcal J_{R^*}$ and $\mathcal J_{\text{LB}}$.

\textbf{Step 1} ($\mathcal J_{\text{SP}}$ to $\mathcal J_{\text{LB}}$).
By triangle inequality,
\[
\left\|\nabla_\theta \mathcal J_{\text{SP}}-\tfrac{\beta}{4}\nabla_\theta \mathcal J_{\text{LB}}\right\|_2
\le
\underbrace{\left\|\nabla_\theta \mathcal J_{\text{SP}}-\nabla_\theta \mathcal J_{R^*}\right\|_2}_{\le\,\delta_{\mathrm{cal}}\ \text{(A2)}}
+
\left\|\nabla_\theta \mathcal J_{R^*}-\tfrac{\beta}{4}\nabla_\theta \mathcal J_{\text{LB}}\right\|_2.
\]
For the second term, since $R^*=\sigma(\beta m)$ and $\sigma'(0)=\tfrac{1}{4}$, a further triangle inequality gives
\[
\left\|\nabla_\theta \mathcal J_{R^*}-\tfrac{\beta}{4}\nabla_\theta \mathcal J_{\text{LB}}\right\|_2
\le
\left\|\nabla_\theta \mathcal J_{R^*}-\tfrac{\beta}{4}\nabla_\theta\mathbb{E}[m]\right\|_2
+
\tfrac{\beta}{4}\left\|\nabla_\theta\mathbb{E}[m]-\nabla_\theta \mathcal J_{\text{LB}}\right\|_2.
\]
By mean-value theorem with A1(a) ($|\beta m|\le M$):
\begin{equation}
\left\|\nabla_\theta \mathcal J_{R^*}-\tfrac{\beta}{4}\nabla_\theta\mathbb{E}[m]\right\|_2
\le\beta L_\sigma(M)M\,B_m.
\label{eq:proof1}
\end{equation}
Since $m=\log p(s^*\mid z)-\log p(\tilde s\mid z)$, we have $\mathbb{E}[m]=\mathcal J_{\text{LB}}(\theta)-\mathbb{E}[\log p(\tilde s\mid z)]$, so by A1(b):
\begin{equation}
\left\|\nabla_\theta\mathbb{E}[m]-\nabla_\theta \mathcal J_{\text{LB}}\right\|_2\le\eta.
\label{eq:proof2}
\end{equation}
Combining:
$\|\nabla_\theta \mathcal J_{\text{SP}}-\tfrac{\beta}{4}\nabla_\theta \mathcal J_{\text{LB}}\|_2
\le\delta_{\mathrm{cal}}+\beta L_\sigma(M)M\,B_m+\tfrac{\beta}{4}\eta.$

\textbf{Step 2} ($\mathcal J_{\text{LB}}$ to $\mathcal J$ via variational identity).
For each $(x,s^*,\mathcal D_{\text{ref}})$, letting $p_\theta\propto\pi_\theta\,p(s^*\mid z)$:
\[
\mathcal J(\theta)=\mathcal J_{\text{LB}}(\theta)+\mathbb{E}\!\left[\mathrm{KL}\!\left(\pi_\theta\;\|\;p_\theta\right)\right],
\]
so $\mathcal J(\theta)\ge \mathcal J_{\text{LB}}(\theta)$; by A3: $\|\nabla_\theta \mathcal J_{\text{LB}}(\theta)-\nabla_\theta\mathcal J(\theta)\|_2\le\epsilon_{\mathrm{vi}}$.

\textbf{Step 3} (Empirical gap and conclusion).
By A3 and triangle inequality,
\[
\left\|\nabla_\theta\widehat{\mathcal J}_{\text{SP}}(\theta,\phi)-\tfrac{\beta}{4}\nabla_\theta\mathcal J(\theta)\right\|_2
\le
\epsilon_{\mathrm{emp}}
+\left\|\nabla_\theta \mathcal J_{\text{SP}}(\theta,\phi)-\tfrac{\beta}{4}\nabla_\theta \mathcal J_{\text{LB}}(\theta)\right\|_2
+\tfrac{\beta}{4}\epsilon_{\mathrm{vi}}.
\]
Substituting Step 1 yields \eqref{eq:surrogate_bound}.
\end{proof}

\subsection{Details of Informal Theorem 2 (Variance Reduction)}
\label{app:theory:thm2_details}

We analyze planner gradient variance with critique-conditioned training.
For the critique-conditioned policy objective
\[
\mathcal J_{\mathrm{pg}}(\theta):=\mathbb E_{z\sim\pi_\theta(\cdot\mid x,\mathcal D_{\text{ref}},c),\,r\sim p(r\mid z,c,x)}[r],
\]
the REINFORCE identity gives
\[
\nabla_\theta \mathcal J_{\mathrm{pg}}(\theta)=\mathbb E\!\left[H(z,c;x)\,(r-b(x,c))\right],
\]
which motivates the end-to-end estimator below; replacing $r$ by $\bar r(z,c,x)=\mathbb E[r\mid z,c,x]$ gives the two-agent estimator.
Let
\[
z\sim \pi_\theta(\cdot\mid x,\mathcal D_{\text{ref}},c),
\quad
H(z,c;x):=\nabla_\theta \log \pi_\theta(z\mid x,\mathcal D_{\text{ref}},c),
\]
and let $r\sim p(r\mid z,c,x)$ denote reward obtained through black-box execution.

Define end-to-end estimator
\begin{equation}
\hat g_{\text{e2e}}:=H(z,c;x)\,(r-b(x,c)).
\label{eq:ge2e_c}
\end{equation}
Define two-agent conditional-mean estimator
\begin{equation}
\bar r(z,c,x):=\mathbb E[r\mid z,c,x],\qquad
\hat g_{\text{2a}}:=H(z,c;x)\,(\bar r(z,c,x)-b(x,c)).
\label{eq:g2a_c}
\end{equation}

\paragraph{Assumptions for Theorem 2.}
\begin{assumption}[B1 (Finite second moments)]
$\mathbb E[\|H(z,c;x)\|_2^2] < \infty$ and $\mathbb E[r^2\mid z,c,x] < \infty$.
\end{assumption}

\begin{assumption}[B2 (Baseline independence)]
$b(x,c)$ is independent of $r$ conditioned on $(x,c)$.
\end{assumption}

\begin{theorem}[Variance Decomposition]
\label{thm:var_decomp_2agent}
Under B1--B2,
\begin{equation}
\mathrm{Var}(\hat g_{\text{e2e}})
=
\mathrm{Var}(\hat g_{\text{2a}})
+
\Delta_{\text{exec}\mid c},
\label{eq:var_decomp_2a}
\end{equation}
where
\begin{equation}
\Delta_{\text{exec}\mid c}
:=
\mathbb E\!\left[
\|H(z,c;x)\|_2^2\,\mathrm{Var}(r\mid z,c,x)
\right]\ge 0.
\label{eq:delta_exec_c}
\end{equation}
Hence $\mathrm{Var}(\hat g_{\text{2a}})\le \mathrm{Var}(\hat g_{\text{e2e}})$,
with strict inequality whenever $\mathrm{Var}(r\mid z,c,x)>0$ on a set of non-zero measure.
This is the formal version of Informal Theorem~2 in the main text.
\end{theorem}

\begin{proof}
Apply law of total variance conditioning on $(z,c)$:
\[
\mathrm{Var}(Y)=\mathbb E[\mathrm{Var}(Y\mid z,c)]
+\mathrm{Var}(\mathbb E[Y\mid z,c]).
\]
Take $Y=\hat g_{\text{e2e}}$.
Conditioned on $(z,c)$, only $r$ is random:
\[
\mathrm{Var}(\hat g_{\text{e2e}}\mid z,c,x)
=
\|H(z,c;x)\|_2^2\,\mathrm{Var}(r\mid z,c,x).
\]
Taking expectation gives $\Delta_{\text{exec}\mid c}$.
Also,
\[
\mathbb E[\hat g_{\text{e2e}}\mid z,c,x]
=
H(z,c;x)\,(\bar r(z,c,x)-b(x,c))
=
\hat g_{\text{2a}}.
\]
Thus the second term is $\mathrm{Var}(\hat g_{\text{2a}})$, proving \eqref{eq:var_decomp_2a}.
\end{proof}

\paragraph{Interpretation.}
End-to-end optimization carries executor-induced randomness directly into policy gradient. The two-agent decomposition conditions on critique signal and replaces stochastic return with a lower-noise conditional target, removing the additive variance term in \eqref{eq:delta_exec_c}.

\section{More Technical Details}

\subsection{Details about reference-based PSP Data}
\label{app:train-data}

We present more details about data curation.

\subsubsection{Instruction-Target-Reference Triplets.}
We detail the construction of the data triplets. Let's use the training data $\mathcal{D}_{\text{tr}}$ as an example. 
By definition, each data point in
\begin{align*}
\mathcal{D}_{\text{tr}} = \{(x^{(1)}, (s^{*})^{(1)}, \mathcal{D}_{\text{ref}}^{(1)}), \ldots, (x^{(J)}, (s^{*})^{(J)}, \mathcal{D}_{\text{ref}}^{(J)})\}
\end{align*}
is a triplet consisting of three components: (i) a high-level instruction $x$, (ii) a gold slide page $s^*$ to be generated, and (iii) a reference set $\mathcal{D}_{\text{ref}}$. 
In practice, one may only have access to raw slide decks from users. Nevertheless, these triplets can be readily constructed from such decks as follows:

\begin{itemize}
\item 
\textbf{Gold Slide Page.} 
Each individual slide page $s$ within the training split of deck is treated as a target page $s^*$ to be generated. 

\item
\textbf{High-level Instruction.} 
We employ a capable VLM (e.g., GPT-4o) to provide a \emph{concise} summary of the slide page $s^*$, e.g., naming the text boxes and visual elements. 
Note that this summary captures only the factual content without describing visual styles such as coloring, font size, or layout organization, thereby serving the role of the high-level instruction $x$. 
Notably, these omitted visual styles represent the latent design intent that must be inferred from the reference set for successful PSP.

\item
\textbf{Reference Set.} 
We randomly pair each $s^*$ with a few other pages from the same deck to serve as the $\mathcal{D}_{\text{ref}}$, leveraging the inherent visual coherence within a single deck. In our experiments, we set the size of the reference set to 1. 
We use a single reference slide to create a challenging personalization setting while keeping the retrieval protocol consistent across methods.

\end{itemize}

\noindent
This procedure transforms raw slide decks into the structured training corpus $\mathcal{D}_{\text{tr}}$. $\D_\text{te}$ is constructed similarly.

\subsubsection{Train, Test, and OOD Data.}
As presented in Sec.~\ref{sec:exp}, we first perform a slide-level held-out split that preserves the deck's chronological order. The last 20\% of slides are reserved for testing, and the remaining slides are used for training (or for retrieval in training-free baselines).

Next, we construct target--reference pairs within each split to ensure that test slides are never leaked during training. To determine the feasibility of each pair, we follow \cite{tang2025slidecoder} to compute the visual complexity scores for both target and reference slides, filtering out pairs where the target--reference complexity score gap is excessive. 
When constructing these pairs, we consider any combination of two slides within a split to be viable, regardless of their original chronological order. Consequently, for a filtered split containing $K$ pages, the resulting dataset size is $K(K-1)$.

To better reflect realistic PSP requirements, we construct OOD data using a more challenging protocol. 
These OOD decks are never seen during training; starting from the second page, each slide is paired with its \emph{immediately preceding} page. 
This setup mimics the sequential nature of real-world slide creation. No complexity filtering is applied to the OOD set. Therefore, for a raw OOD deck of $K$ pages, the final evaluation set is of size $K-1$.

\subsection{Training Details of the Critic}
\label{app:critic}
We provide additional training details for the critic agent.

\subsubsection{Critic Prompt.}
We prompt the critic to act as a presentation design expert responsible for evaluating a generated slide against the provided user instruction $x$ and reference slide(s) $\mathcal{D}_{\text{ref}}$. 
The critique generation is structured into three stages.
We provide the prompt template in the end of the supplementary material. 
\begin{itemize}
\item 
\textbf{Analysis.} The critic first writes a brief high-level analysis summarizing the major design inconsistencies. 

\item
\textbf{Assessment.}
It then assesses the generated slide along eight predefined structural aspects subject to random perturbations:
\emph{graphic color, graphic position, graphic size, image position, image size, text color, text position}, and \emph{text size}. See Supp.~\ref{app:perturb} for further details on the perturbation mechanism.

\item 
\textbf{Feedback.}
Finally, the critic outputs a single structured block containing aspect-wise feedback for all eight aspects in a fixed order. Each entry includes the current issue and a target correction. To ensure the feedback is directly usable for downstream refinement, the prompt requires concise and actionable suggestions, with explicit numeric targets in normalized coordinates whenever position or size adjustments are necessary.
\end{itemize}

\subsubsection{Format Reward.}
We employ a rule-based format reward $R_{\text{format},c}(c)$ to encourage the critic to follow the required multi-stage output schema. The verification checks whether each stage is properly produced, including the \emph{overall} structure as well as the specific requirements for the \emph{analysis}, \emph{assessment}, and final aspect-wise \emph{feedback} sections.
The verification checks: (i) compliance with the required format by identifying keyword tags and ensuring the content is non-empty; and (ii) whether the names of aspect-level entries are valid and non-duplicated. 

Mathematically, the format reward is a weighted sum of these components:
\begin{align*}
R_{\text{format, c}}(c)
&=
0.05\,R_{\text{structure}}(c)
+
0.05\,R_{\text{analysis}}(c) \\
&\quad +
0.05\,R_{\text{assessment}}(c)
+
0.85\,R_{\text{feedback}}(c).
\end{align*}
We set $R_{\text{feedback}}(c)$ as the dominant term, reflecting that the structured aspect-wise feedback is the most critical part of the critic's output.

\subsubsection{Accuracy Reward.}
The accuracy reward, as introduced in Sec.~2.2, verifies whether the extracted critique is consistent with the ground-truth discrepancy list $\mathcal{A}_{\text{diff}}$. 
Specifically, for each discrepancy item in $\mathcal{A}_{\text{diff}}$, we check whether the critic correctly diagnoses both the issue and the corresponding correction. This verification is implemented via a rule-based keyword matching procedure between the \emph{extracted} critique and the perturbation annotations. We compute the final reward by aggregating the matched issue--correction keywords over all discrepancy items:
\begin{equation}
R_{\text{vfy}} (c, \mathcal{A}_{\text{diff}})
=
\frac{\text{\# matched issue--correction keywords in } (c, \mathcal{A}_{\text{diff}})}{\text{\# all ground-truth issue--correction keywords in }\mathcal{A}_{\text{diff}}}.
\end{equation}
Essentially, the accuracy reward measures the precision with which the critique covers the issue--correction pairs specified by $\mathcal{A}_{\text{diff}}$.

This reward, combined with the format reward, constitutes the final reward for critic training defined in Eq.~\eqref{eq:critic_obj}.

\subsection{Training Details of the Planner}
\label{app:plnr}
We provide additional training details for the planner agent.

\subsubsection{Planner Prompt.}
We prompt the planner to act as an expert slide design planner responsible for producing a detailed and executable design plan conditioned on the provided user instruction $x$ and reference slide(s) $\mathcal{D}_{\text{ref}}$. 
As detailed in Sec.~2.2, the planner aims to perform two complementary subtasks: 
(i) to generate a high-quality initial plan based on the user instruction and reference slides; and 
(ii) to revise a suboptimal plan if a critique is provided. 
We use separate prompt templates for these two scenarios, both of which are structured into three stages. 
We provide the prompt template at the end of the supplementary material. 

\paragraph{Initial Plan Generation.}
\begin{itemize}
\item
\textbf{Analysis.}
The planner first infers transferable design principles from the reference slide(s), including the core visual logic, reusable stylistic elements, and context-specific factors that should be adapted to the current request.

\item
\textbf{Strategy.}
It then formulates an adaptation strategy that maps these principles into concrete planning decisions, including the overall layout structure, the roles of major visual elements, and the aspect-level choices governing color, position, size, and typographic hierarchy.

\item
\textbf{Plan.}
Finally, the planner outputs a complete design plan that specifies the slide composition in a sequential and directly executable manner, covering the background, major containers, textual components, visual elements, and overall spacing/balance. To ensure executability, the prompt requires concrete normalized spatial specifications and prohibits ambiguous instructions.
\end{itemize}

\paragraph{Critique-Guided Plan Revision.}
\begin{itemize}
\item
\textbf{Analysis.}
The planner first infers transferable design principles from the reference slide(s), mimicking a chain-of-thought process.

\item
\textbf{Strategy.}
It then revises the suboptimal plan according to the provided critique $c$. In particular, the planner first normalizes the previous plan into an explicit structured representation, and then applies the required edits while preserving aspects that are already satisfactory.

\item
\textbf{Plan.}
Finally, the planner outputs a new complete design plan that retains the overall structure of the suboptimal plan whenever possible, while incorporating the requested corrections in a directly executable form.
\end{itemize}

Both initial and revised plans are evaluated from \emph{format} and \emph{accuracy} perspectives in a unified way, as detailed below. 

\subsubsection{Format Reward.}
We employ a rule-based format reward $R_{\text{format},p}(z)$ to encourage the planner to follow the required multi-stage output schema. The verification checks whether each stage is properly produced, including the \emph{overall} structure as well as the specific requirements for the \emph{analysis}, \emph{strategy}, and final \emph{plan} sections. The verification checks: (i) compliance with the required format by identifying keyword tags and ensuring the content is non-empty; and (ii) whether the final plan preserves the required asset content, ensuring no critical user-provided material is omitted.

Mathematically, the format reward is a weighted sum of these components:
\begin{align*}
R_{\text{format},p}(z)
&=
0.05\,R_{\text{structure}}(z)
+
0.05\,R_{\text{analysis}}(z) \\
&\quad +
0.10\,R_{\text{strategy}}(z)
+
0.80\,R_{\text{plan}}(z).
\end{align*}
We set $R_{\text{plan}}(z)$ as the dominant term, reflecting that the final executable design plan is the most critical part of the planner's output.

\subsubsection{Accuracy Reward.}
The accuracy reward, as introduced in Sec.~2.2, verifies whether the planner's output is consistent with the target slide under the plan--visual matching (PVM) objective. Specifically, we extract the natural-language design plan from the final \emph{plan} stage and evaluate whether it matches the gold slide $s^*$ more faithfully than a randomly perturbed slide $\tilde{s}$. To mitigate positional bias, the comparison is conducted under both visual presentation orders, and the final reward is computed by averaging the two binary outcomes:
\begin{equation}
R_{\text{pvm}}(z, s^*, \tilde{s})
=
\frac{1}{2}
\left(
\mathbb{I}[\hat{o}_{\rightarrow}=1]
+
\mathbb{I}[\hat{o}_{\leftarrow}=1]
\right).
\end{equation}
Essentially, the accuracy reward measures whether the final plan describes the gold slide more faithfully than the perturbed alternative, based on the judgment of a capable VLM judge (e.g., o4-mini).

This reward, combined with the format reward, constitutes the final reward for planner training defined as
{
\small
\begin{align*}
\mathcal{J}_{\text{plnr}}(\theta)
&= \mathbb{E}_{\substack{s^* \sim \mathcal{D}_{\text{tr}}\\ \tilde{s} \sim q(\cdot \mid s^*)}}
\Big[
\mathbb{E}_{c^* \sim \text{Oracle}(c)}
\big[
\mathbb{E}_{z \sim \pi_\theta(\cdot \mid x, \mathcal{D}_{\text{ref}}, c^*)}
[
R_{\text{pvm}}(z, s^*, \tilde{s}) \times R_\text{format,p}(z)
]
\big]
\Big].
\end{align*}
}

\subsubsection{Synthesis of Critique-Guided Plan Revision Data.}
As detailed in Sec.~\ref{sec:method:spire}, critique-guided revision training relies on high-quality revision triplets $(\tilde{z}, \tilde{s}, c^*)$. Here, $\tilde{s}$ is a perturbed slide, $\tilde{z}$ is a suboptimal plan aligned with $\tilde{s}$, and $c^*$ is an oracle critique that specifies how the suboptimal plan should be revised toward the gold slide $s^*$. The construction process is summarized as follows:

\begin{itemize}
\item
\textbf{Perturbed Slide.}
The perturbed slide $\tilde{s}$ is obtained by applying recorded structural perturbations to the gold slide $s^*$, as described in Sec.~2.2. Since these perturbations are known and explicitly recorded, $\tilde{s}$ serves as a controllable visual target corresponding to the suboptimal state.

\item
\textbf{Oracle Critique.}
To construct the oracle critique $c^*$, we provide the perturbation action list to a capable LLM and ask it to translate the known discrepancies into the required critique format. This converts the recorded structural corruptions into natural-language feedback that is directly consumable by the planner. We further apply rejection filtering using the critique reward introduced in Supp~\ref{app:critic}, retaining only those pass this verification step.

\item
\textbf{Suboptimal Plan.}
To construct the suboptimal plan $\tilde{z}$, we provide the perturbed slide $\tilde{s}$ to a capable LLM and ask it to infer a complete design plan that accurately \emph{reflects the current degraded slide} rather than the gold slide. We then apply rejection filtering based on the planner reward and retain only those candidate plans where the plan--visual matching reward prefers $\tilde{s}$ over $s^*$, thereby enforcing consistency between the inferred plan and the perturbed visual input.
\end{itemize}

Overall, this pipeline ensures coherent revision supervision: $\tilde{s}$ provides the observable suboptimal slide, $\tilde{z}$ captures the corresponding suboptimal design intent, and $c^*$ specifies the corrections needed to recover the intended gold design.

\subsection{Implementation and Evaluation Details of Experiments}
\label{app:experiment}

\subsubsection{Implementation Details.}
We train {\name} with \texttt{verl}~\cite{sheng2024hybridflow}; the detailed training hyper-parameters are provided in Tab~\ref{tab:hparams}. 

\begin{table*}[t!]
\caption{
Hyper-parameters used for critic and planner training.
}
\label{tab:hparams}

\centering
\resizebox{0.85\linewidth}{!}{%
\renewcommand{\tabcolsep}{4pt}
\begin{tabular}{
r
cc
}
\toprule[0.4ex]
\textbf{Parameter} & \textbf{Critic} & \textbf{Planner} \\
\midrule[0.2ex]
Policy model & Qwen2.5-VL-7B-Instruct & Qwen2.5-VL-7B-Instruct \\
RL algorithm & DAPO & DAPO \\
Number of epochs & 1 & 1 \\
Batch size & 4 & 4 \\
Learning rate & $1 \times 10^{-6}$ & $1 \times 10^{-6}$ \\
Maximum prompt length & 6000 & 6000 \\
Maximum response length & 2048 & 2048 \\
Rollout numbers & 4 & 4 \\
Lower clipping ratio & 0.2 & 0.2 \\
Upper clipping ratio & 0.28 & 0.28 \\
Clipping coefficient & 10.0 & 10.0 \\
Entropy coefficient & 0 & 0 \\
KL regularization coefficient & 0 & 0 \\
\bottomrule[0.4ex]
\end{tabular}
}
\end{table*}

\subsubsection{Evaluation Details (Visual Similarity).}
For the visual similarity evaluation, we treat each slide as a rendered full-page image. 
We report two full-reference image similarity metrics, following \cite{tang2025slidecoder}.
The scores are normalized to $[0,1]$.
\begin{itemize}
\item 
\textbf{Structural Similarity Index Measure (SSIM).} SSIM measures the low-level structural fidelity between the generated slide and the gold slide.

\item 
\textbf{CLIP Image Similarity.} CLIP compares the global visual representations of two images in a pre-trained vision-language embedding space. We use AltCLIP~\cite{chen2023altclip}. 
\end{itemize}

\subsubsection{Evaluation Details (VLM-as-a-Judge).}
We also report a reference-free VLM-as-a-judge evaluation protocol. 
In this setting, the VLM judge observes the user instruction $x$ and the generated slide to assess the generation quality across the following four aspects. Each aspect is scored on a 0--10 scale and then normalized to $[0,1]$ for aggregation.
\begin{itemize}
\item
\textbf{Faithfulness.}
Whether the textual content requested in the user instruction is fully preserved and readable, without missing text, content truncation, or severe element overlap.

\item
\textbf{Color.}
The harmony of the color palette, and the contrast between foreground and background elements to ensure readability.

\item
\textbf{Layout.}
The overall composition of the slide, including alignment, spacing, margins, and the logical organization of major visual elements.

\item
\textbf{Overall Aesthetics.}
The overall professionalism and visual appeal of the slide as a presentation-grade page.
\end{itemize}

We provide the full VLM-as-a-judge prompt at the end of the supplementary material.

\section{Additional Experiment Results}
\label{app:results}

We present more experiment results that are not included in the main body due to page limits.

\subsubsection{Reference-based Pairwise VLM-Judge Evaluation.}
\label{app:pairwise-result}
To complement the reference-free evaluation, we conduct a reference-based pairwise comparison where the VLM judge is provided with the user instruction $x$, the gold slide $s^*$, and two generated slides from competing methods. This protocol directly measures relative generation quality against the intended target, making it more sensitive to subtle differences in design execution.

We evaluate each pair along four aspects: (i) \emph{faithfulness} (text preservation and readability); (ii) \emph{color}; (iii) \emph{layout}; and (iv) \emph{overall aesthetics}. 
The definitions of these aspects are the same as in the reference-free setting presented above. 
To mitigate positional bias, we evaluate each pair twice by swapping the presentation order and aggregate the results.

Tab~\ref{tab:vlm_pairwise_winrate} reports the \emph{win/tie rate} of {\name}, representing the fraction of cases where {\name} is judged \emph{at least as good as the baseline}. 
As shown in Tab~\ref{tab:vlm_pairwise_winrate}, {\name} consistently outperforms all baselines on both test and OOD pages. The advantages are particularly pronounced in Layout and Faithfulness, where {\name} achieves \emph{win/tie rates} exceeding 90\% against most baselines. Even against the strong GPT-based PSP (o4-mini) baseline, {\name} maintains a significant lead (avg. 0.847/0.851), confirming that our RL-based training produces slides that are consistently preferred in head-to-head comparisons.

\begin{table}[t!]
\caption{
Reference-based Pairwise VLM Evaluation (Win/Tie Rate of {\name}).
}
\label{tab:vlm_pairwise_winrate}
\centering
\resizebox{0.7\linewidth}{!}{%
\renewcommand{\tabcolsep}{4pt}
\begin{tabular}{
r
cccc
>{\columncolor{gray!10}}c
}
\toprule[0.4ex]
\multirow{2}{*}{\bf Versus} & \multicolumn{5}{c}{\bf VLM-Judge Preference (Win/Tie Rate of {\name})} \\
\cmidrule[0.2ex](lr){2-6}
& Faith $\uparrow$ & Color $\uparrow$ & Layout $\uparrow$ & Aest $\uparrow$ & \bf AVG \\
\midrule[0.2ex]
\multicolumn{6}{c}{\textbf{Test Pages}} \\
\midrule[0.2ex]
AutoPresent~\cite{ge2025autopresent} & 0.812 & 0.660 & 0.757 & 0.727 & 0.739 \\
PPTAgent~\cite{zheng2025pptagent} & 0.981 & 0.921 & 0.972 & 0.966 & 0.960 \\
SD 3.5~\cite{esser2024scaling} & 0.932 & 0.881 & 0.939 & 0.902 & 0.914 \\
PSP (o4-mini) & 0.899 & 0.832 & 0.801 & 0.856 & 0.847 \\
PSP (Base) & 0.995 & 0.859 & 0.961 & 0.951 & 0.942 \\
\midrule[0.2ex]

\multicolumn{6}{c}{\textbf{OOD Pages}} \\
\midrule[0.2ex]
AutoPresent~\cite{ge2025autopresent} & 0.896 & 0.761 & 0.791 & 0.763 & 0.803 \\
PPTAgent~\cite{zheng2025pptagent} & 0.983 & 0.931 & 0.934 & 0.925 & 0.943 \\
SD 3.5~\cite{esser2024scaling} & 0.941 & 0.899 & 0.898 & 0.854 & 0.898 \\
PSP (o4-mini) & 0.911 & 0.832 & 0.851 & 0.809 & 0.851 \\
PSP (Base) & 0.951 & 0.903 & 0.911 & 0.887 & 0.913 \\
\bottomrule[0.4ex]
\end{tabular}
}
\end{table}

\subsection{More Qualitative Results}
\label{supp:qualitative}

Fig~\ref{fig:main_more} presents additional visual comparisons between {\name} and baseline methods across both in-distribution test slides and OOD scenarios. 
Consistent with the observations in Sec.~\ref{sec:exp:qual}, {\name} demonstrates superior faithfulness to the user instruction while maintaining high design consistency with the reference slides.

\ExplSyntaxOn
\cs_set:Npn \loadimage #1#2#3#4 {
    \seq_clear:N \l_tmpa_seq
    \int_step_inline:nnn {#3} {#3 + #4 - 1} {
        \file_if_exist:nTF {#1/##1.pdf} 
        {
            \seq_put_right:Nx \l_tmpa_seq {
                \exp_not:N \adjustbox {valign=m} {
                    \exp_not:N \parbox [c] [\exp_not:N \figheight] [c] {\exp_not:N \figwidth} {
                        \exp_not:N \centering
                        \exp_not:N \includegraphics [max~width=\exp_not:N \figwidth, max~height=\exp_not:N \figheight, keepaspectratio, \exp_not:n {#2}] {#1/##1.pdf}
                    }
                }
            }
        }
        {
            \seq_put_right:Nx \l_tmpa_seq {
                \exp_not:N \adjustbox {valign=m} {
                    \exp_not:N \parbox [c] [\exp_not:N \figheight] [c] {\exp_not:N \figwidth} {
                        \exp_not:N \centering
                        \exp_not:N \includegraphics [max~width=\exp_not:N \figwidth, max~height=\exp_not:N \figheight, keepaspectratio, \exp_not:n {#2}] {example-image-a}
                    }
                }
            }
        }
    }
    \seq_use:Nn \l_tmpa_seq {&}
}
\ExplSyntaxOff

\begin{figure*}[htb!] 
\centering
\renewcommand{\tabcolsep}{1.3pt}
\def\figwidth{0.11\textwidth} 
\def\figheight{1.2cm}

\newcommand{\rowgroup}[1]{
    \adjustbox{rotate=90}{{#1}}
}

\begin{tabular}{c c ccccc cc} 
\toprule[0.4ex]

& \scriptsize Ref 
& \scriptsize AutoPre. 
& \scriptsize PPTAgent 
& \scriptsize SD 3.5 
& {\scriptsize \begin{tabular}[c]{@{}c@{}}PSP\\(o4-mini)\end{tabular}} 
& {\scriptsize \begin{tabular}[c]{@{}c@{}}PSP\\(Base)\end{tabular}} 
& \scriptsize Ours 
& \scriptsize Gold \\
\noalign{\vskip 0.5ex}
\midrule[0.2ex]

\multirow{2}{*}[-0.85cm]{\rowgroup{Test}} 
& \loadimage{figures/8_s_137_8}{width=\figwidth, height=\figheight, keepaspectratio}{0}{8} \\ 
& \loadimage{figures/29_s_129_6}{width=\figwidth, height=\figheight, keepaspectratio}{0}{8} \\ 
& \loadimage{figures/47_s_146_9}{width=\figwidth, height=\figheight, keepaspectratio}{0}{8} \\

\noalign{\vskip 0.2ex}\cdashline{2-9}\noalign{\vskip 0.2ex}
\multirow{2}{*}[-0.45cm]{\rowgroup{OOD}} 
& \loadimage{figures/entrepreneur_3}{width=\figwidth, height=\figheight, keepaspectratio}{0}{8} \\
& \loadimage{figures/technology_9}{width=\figwidth, height=\figheight, keepaspectratio}{0}{8} \\ 

\bottomrule[0.4ex]
\end{tabular}
\caption{
Additional qualitative comparison of slide generation results.
}
\label{fig:main_more}
\end{figure*}

\section{More Details about Perturbation}
\label{app:perturb}

We construct perturbations using a structured taxonomy defined along two axes:
\textbf{shape role} and \textbf{visual attribute}. Each slide element is
first assigned one of three shape roles, namely \textit{text}, \textit{graphic},
or \textit{image}, according to its functional role in the slide. Conditioned on
this role assignment, we define a fixed set of perturbable attributes, resulting
in \textbf{eight perturbation categories}.

\subsection{Perturbation Categories}

\paragraph{Text-related categories.}
For text elements, we define three perturbation categories:
\begin{itemize}
    \item \textbf{Text Size} (\texttt{text\_size}): perturbs the font size of a
    text portion.
    \item \textbf{Text Position} (\texttt{text\_position}): perturbs text layout
    by replacing the paragraph alignment with a different valid alignment.
    \item \textbf{Text Color} (\texttt{text\_color}): perturbs the font color of a
    text element.
\end{itemize}

\paragraph{Graphic-related categories.}
For non-image graphic shapes, we define three perturbation categories:
\begin{itemize}
    \item \textbf{Graphic Size} (\texttt{graphic\_size}): perturbs one size
    attribute of a graphic element, namely width or height.
    \item \textbf{Graphic Position} (\texttt{graphic\_position}): perturbs one
    spatial coordinate of a graphic element, namely its horizontal coordinate
    $x$ or vertical coordinate $y$, where $x$ and $y$ denote the horizontal and
    vertical positions of the element on the slide canvas, respectively.
    \item \textbf{Graphic Color} (\texttt{graphic\_color}): perturbs the fill color
    of a graphic element.
\end{itemize}

\paragraph{Image-related categories.}
For image elements, we define two perturbation categories:
\begin{itemize}
    \item \textbf{Image Size} (\texttt{image\_size}): perturbs the width or
    height of an image element.
    \item \textbf{Image Position} (\texttt{image\_position}): perturbs the
    horizontal coordinate $x$ or vertical coordinate $y$ of an image element.
\end{itemize}

\subsection{Perturbation Magnitude}

For numerical attributes, we control perturbation strength with a discrete
severity level. Let $a \in \mathbb{R}$ denote the original value of a numerical
attribute and let $a'$ denote the perturbed value. We sample a relative
perturbation ratio $\delta$ as
\begin{equation}
    \delta = s \cdot u,
    \qquad
    u \sim \mathcal{U}(\delta_{\min}, \delta_{\max}),
    \qquad
    s \sim \mathrm{Unif}\{-1, +1\},
\end{equation}
where $\mathcal{U}(\delta_{\min}, \delta_{\max})$ denotes the uniform
distribution over the interval $[\delta_{\min}, \delta_{\max}]$,
$\delta_{\min}$ and $\delta_{\max}$ are the minimum and maximum perturbation
magnitudes associated with the selected category and severity level,
$u$ is the sampled perturbation magnitude, and $s$ determines the perturbation
direction. The perturbed value is then computed as
\begin{equation}
    a' = a \left(1 + \frac{\delta}{100}\right),
\end{equation}
where $\delta$ is expressed in percentage. When required, $a'$ is clipped to the
valid range of the corresponding attribute, e.g., $[0,255]$ for RGB channels. For categorical attributes, the perturbed value is sampled uniformly from the
set of valid alternatives excluding the original value.

\subsection{Sampling Strategy}
Given a slide, we first enumerate all perturbable items and group them by the
eight categories above. Each category is then activated independently according
to a Bernoulli random variable with parameter $p$, where $p$ denotes the
probability that a category is selected for perturbation. In our experiments, we
set $p = 0.5$. Categories with no valid candidates are skipped automatically.

For each activated category, the pipeline samples one valid candidate from the
corresponding pool and generates one perturbation instruction. To avoid
conflicting edits, each physical shape is constrained to receive at most one
primary perturbation. After primary perturbations are generated, the pipeline
may further sample one additional extra perturbation from the remaining
candidate pool, typically using a lower severity level. This design increases
perturbation diversity while preserving semantic and geometric consistency at
the slide level.

The resulting perturbation process is role-aware, attribute-specific, and
severity-controlled. It produces diverse but structured modifications that
approximate realistic slide variations, while preventing incompatible edits to
the same element. Figure~\ref{fig:perturb_examples} provides qualitative examples of our
perturbation pipeline. The first column contains the original slides and the second
column contains their perturbed counterparts, with each column corresponding to one
slide-level example. 

\begin{figure*}[t]
    \centering
    \setlength{\tabcolsep}{4pt}
    \renewcommand{\arraystretch}{0}
    \begin{tabular}{cc}
        Original & Perturbed \\[4pt]
        \includegraphics[width=0.47\textwidth]{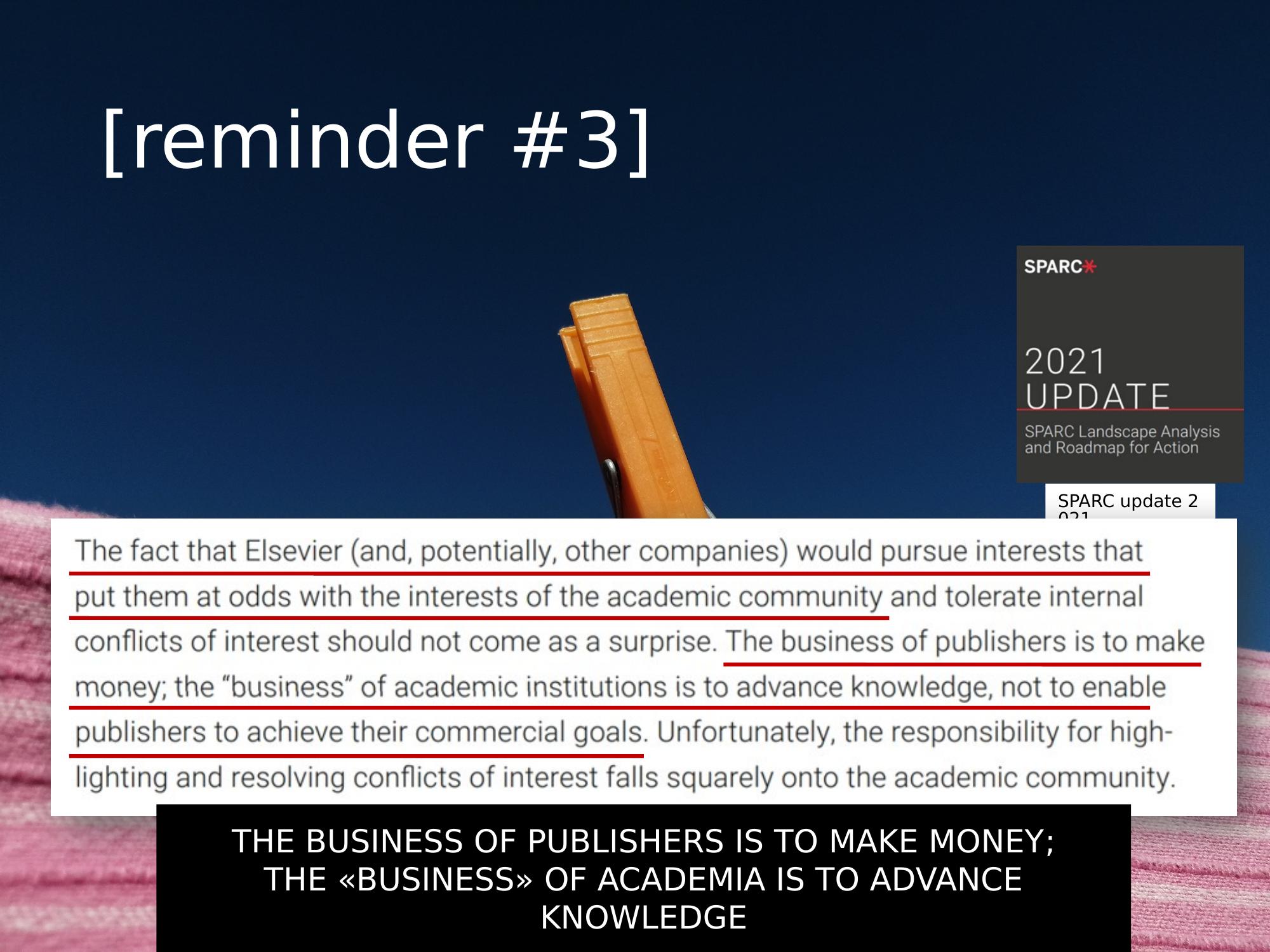} &
        \includegraphics[width=0.47\textwidth]{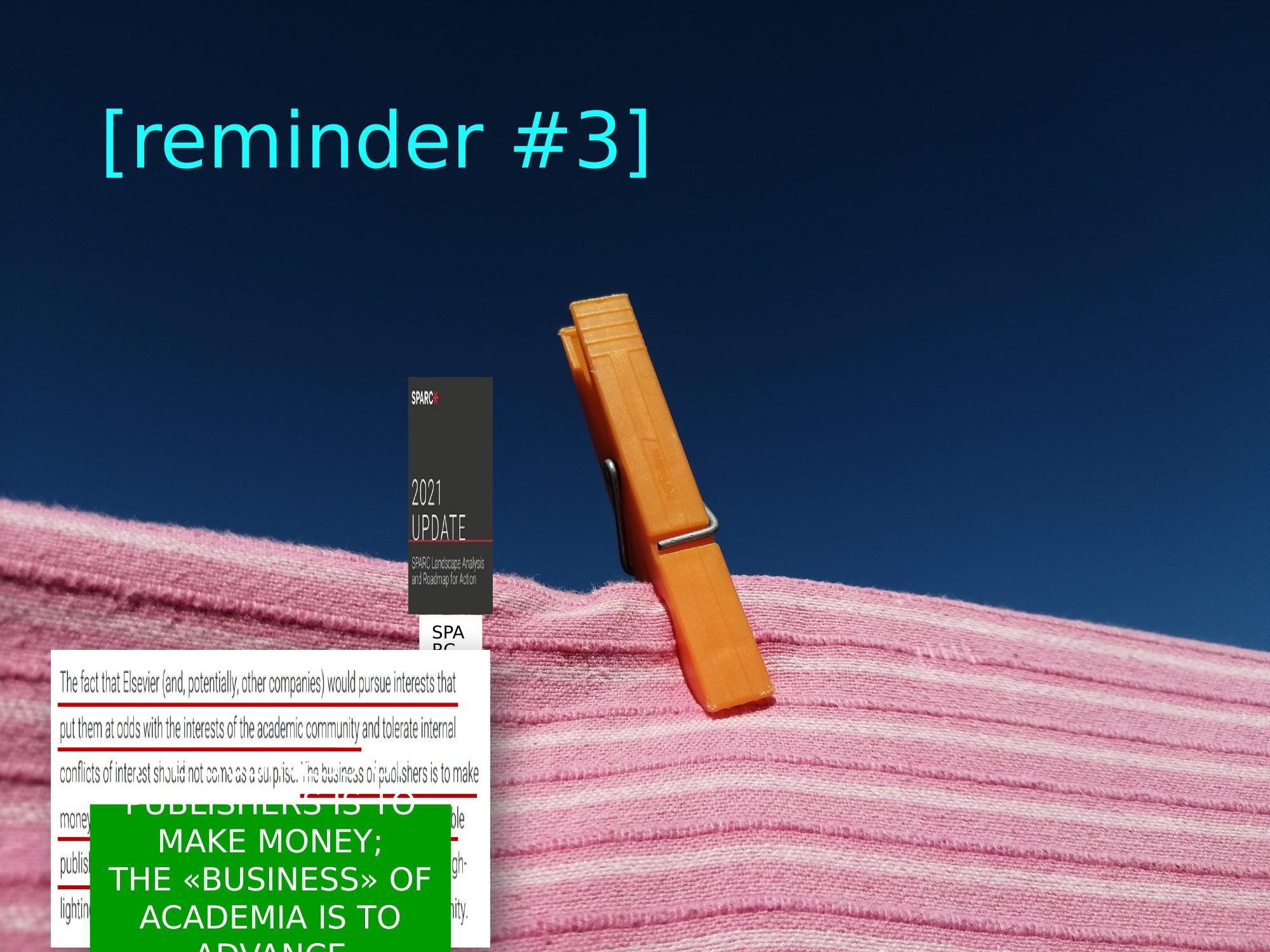} \\[20pt]
        \includegraphics[width=0.47\textwidth]{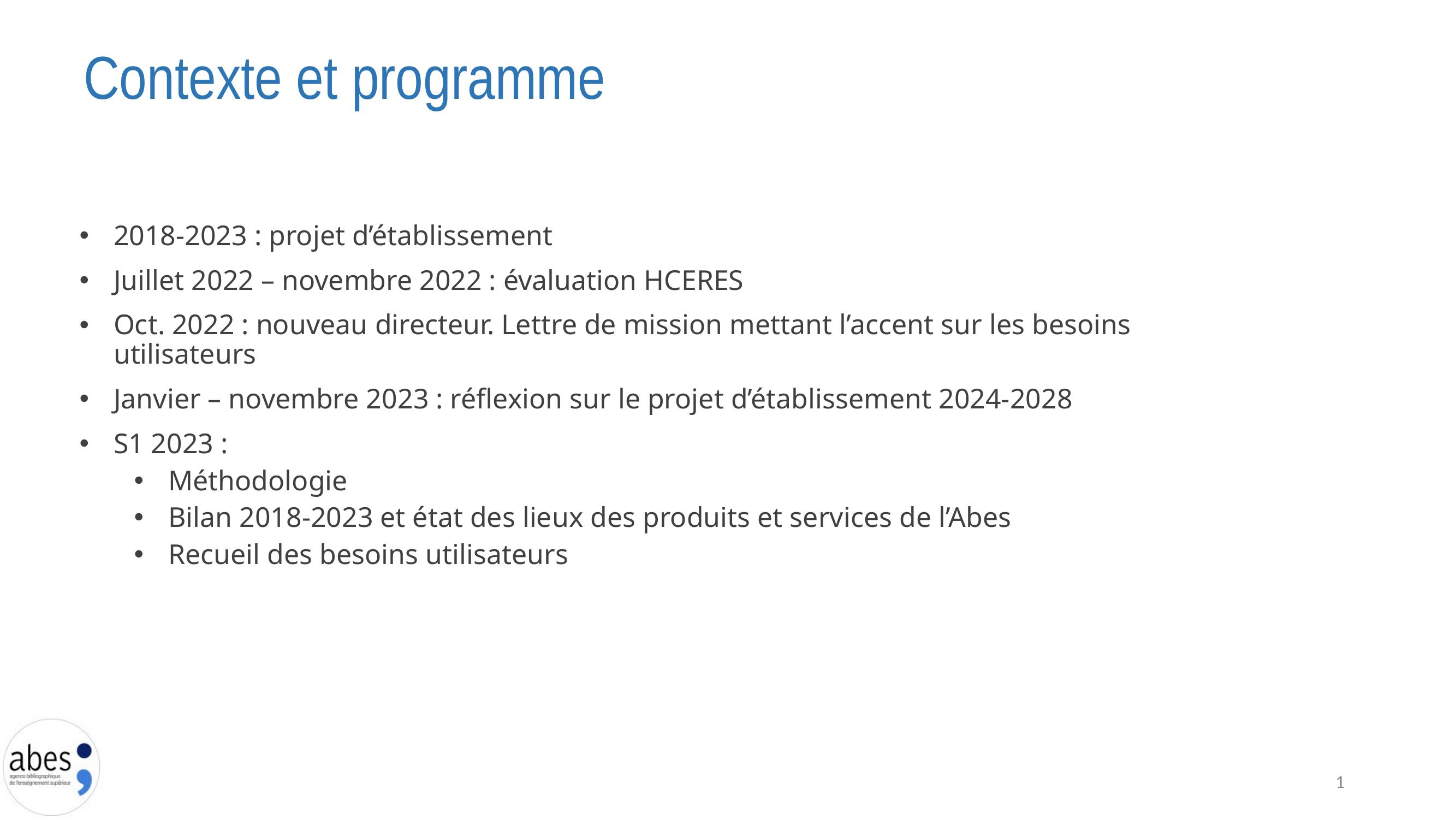} &
        \includegraphics[width=0.47\textwidth]{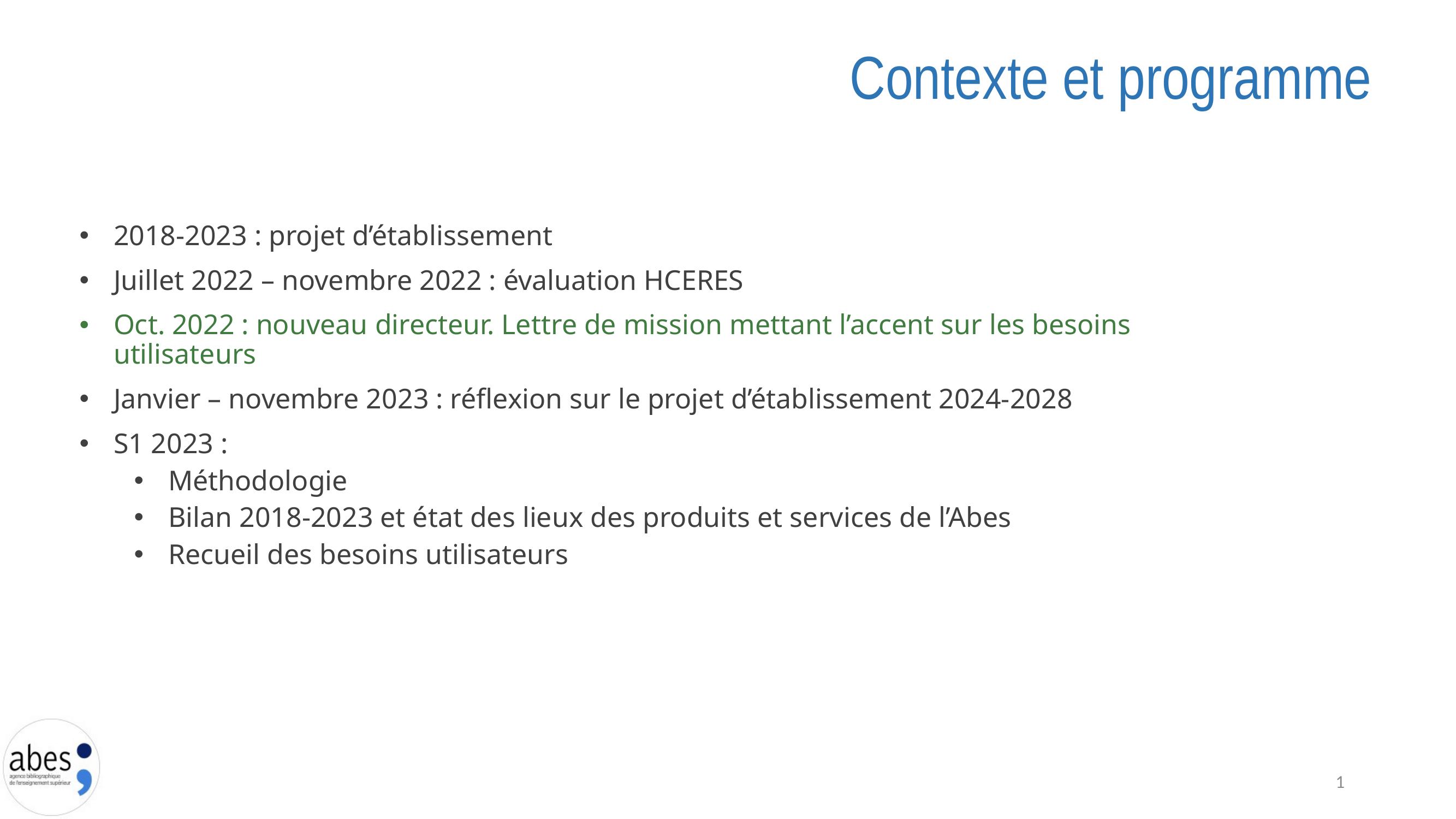} \\[20pt]
        \includegraphics[width=0.47\textwidth]{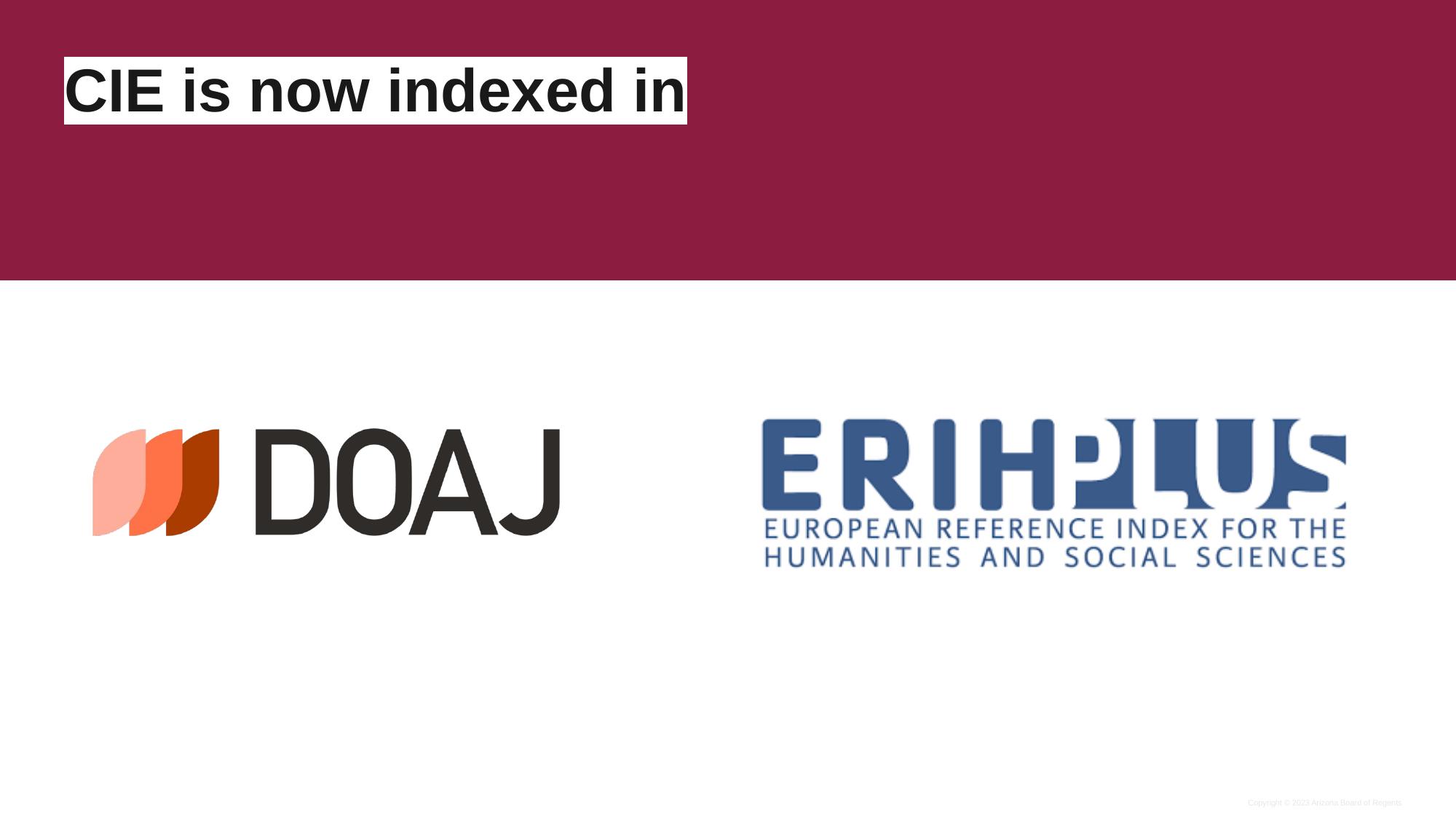} &
        \includegraphics[width=0.47\textwidth]{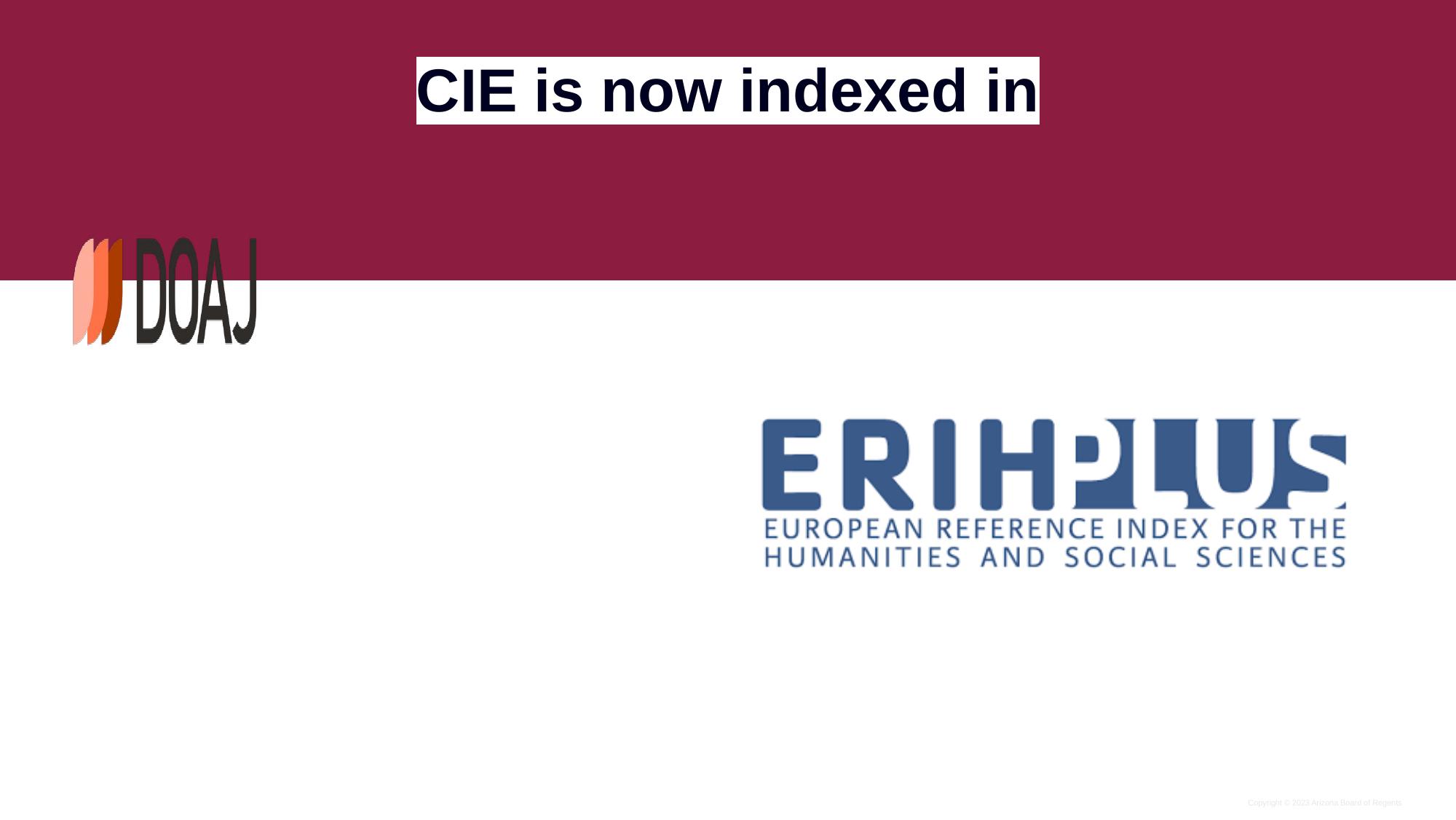} \\
    \end{tabular}
    \caption{
    Qualitative examples of the proposed perturbation pipeline. Each row shows
    one original-perturbed slide pair.
    }
    \label{fig:perturb_examples}
\end{figure*}

\UseRawInputEncoding

\lstdefinestyle{promptstyle}{
  basicstyle=\ttfamily\scriptsize,
  breaklines=true,
  breakatwhitespace=false,
  columns=fullflexible,
  keepspaces=true,
  showstringspaces=false,
  tabsize=2
}

\begin{tcblisting}{
listing only,
breakable,
enhanced,
colback=lightblue!5,
colframe=lightblue!50,
rounded corners,
top=6pt,bottom=6pt,left=8pt,right=8pt,boxsep=4pt,
title=System Prompt for Critique Generation ({\name}),
listing options={style=promptstyle}
}
You are a presentation design expert. Evaluate a generated slide using the reference slide as a STYLE GUIDE (not a template). The goal is visual consistency (style, layout logic, hierarchy), not identical content. Your response must follow this EXACT format:

## STEP 1: ANALYSIS
<analysis>
Write 2-3 sentences summarizing the main consistency issues.
</analysis>


## STEP 2: ELEMENT STYLE ASSESSMENT
<think>
1) GRAPHIC_COLOR: palette + color roles.
2) GRAPHIC_POSITION: alignment + margins + spacing rhythm.
3) GRAPHIC_SIZE: relative visual weight + hierarchy.
4) IMAGE_POSITION: anchoring + alignment + proximity to related text.
5) IMAGE_SIZE: prominence + proportions + image-to-text balance.
6) TEXT_COLOR: contrast + emphasis conventions.
7) TEXT_POSITION: grid/columns + padding + wrap region.
8) TEXT_SIZE: typographic hierarchy + consistency.
</think>


## STEP 3: FEEDBACK CHECKLIST
Output a SINGLE <comment> block containing the feedback for ALL 8 aspects.
Ensure every aspect from Step 2 is included in the list inside the block:
<comment>
<aspect>GRAPHIC_COLOR</aspect>: <current>...</current><target>...</target>
...
<aspect>TEXT_SIZE</aspect>: <current>...</current><target>...</target>
</comment>


Rules:
1) Use proportions only (no px/pt). For position/size revisions, give concrete ratios/percentages and specify the frame (of panel/of slide).
 Prefer bbox x:.. y:.. w:.. h:.. when precision is needed; otherwise provide at least one clear numeric target (e.g., target margin/height/width) plus an alignment rule.
2) Always output ALL 8 aspects, in the exact order above.
3) If acceptable, use <target>GOOD</target>. Otherwise provide a specific revision to improve consistency.
4) Be concise: <current> and <target> should each be 1 short sentence.
5) Avoid vague revisions (e.g., "move a bit", "make bigger"). Always include at least one numeric target.
\end{tcblisting}

\begin{tcblisting}{
listing only,
breakable,
enhanced,
colback=lightblue!5,
colframe=lightblue!50,
rounded corners,
top=6pt,bottom=6pt,left=8pt,right=8pt,boxsep=4pt,
title=User Prompt for Critique Generation ({\name}),
listing options={style=promptstyle}
}
Evaluate the **Generated Slide** using the **Slide Request** below.
Note: The request may include: (1) media image(s) to place in the slide, and (2) style reference slide(s) used as a STYLE GUIDE.
Critique visual consistency with the style reference (not identical content).
--- [SLIDE REQUEST START] ---
{user_prompt}
--- [SLIDE REQUEST END] ---
--- [GENERATED SLIDE START] ---
<image>
--- [GENERATED SLIDE END] ---
Follow the exact 8-aspect output format from the system prompt using <aspect>, <current>, <target> tags.
\end{tcblisting}

\UseRawInputEncoding

\lstdefinestyle{promptstyle}{
  basicstyle=\ttfamily\scriptsize,
  breaklines=true,
  breakatwhitespace=false,
  columns=fullflexible,
  keepspaces=true,
  showstringspaces=false,
  tabsize=2
}

\begin{tcblisting}{
listing only,
breakable,
enhanced,
colback=lightblue!5,
colframe=lightblue!50,
rounded corners,
top=1pt,bottom=1pt,left=2pt,right=2pt,boxsep=0.5pt,
title=System Prompt for Plan Generation ({\name}),
listing options={style=promptstyle}
}
You are an expert slide design planner. The **goal** is to transform brief instructions into a detailed, implementable design plan by intelligently adapting reference slides. Your response must follow this EXACT format:

## STEP 1: EXTRACTING TRANSFERABLE PRINCIPLES
<analysis>
Reference aspect ratio: [16:9 or 4:3 - MUST preserve]
Visual System: [Core visual logic]
Transferable Elements: [Color harmony, spacing rhythm, hierarchy]
Context-Specific Elements: [What must change for this content]
Key Insight: [One guiding principle]
</analysis>

## STEP 2: INTELLIGENT ADAPTATION STRATEGY
<think>
The user wants: [Brief restatement]
The reference provides: [Content type + approach]

**Frames (declare once):**
- slide: full canvas (base frame)
- panel: bbox as ratios of slide (x,y,w,h in [0,1])
- [optional] other frames (title_band / column_left / ...), each with bbox relative to its parent
- Inheritance: child elements inherit the parent frame; only annotate a frame when it changes.

**Image Role Analysis:**
- [image_path]: Role=[Background/Illustration/Logo/...]; Evidence=\[quote/paraphrase]\
- ... (Repeat per image)

**Proportional Sizing Strategy (concise):**
- **Resolve Ambiguity:** If the user's plan provides a range (e.g., 30-40%) or a vague term (e.g., 'about 40%'), 
your job is to **resolve this by choosing a single, concrete value** (e.g., 35% or 40%).

- Default frame = panel. Write `% of <frame>` once per element; omit `of panel` when using the default.
- If width/height use different frames, add: [ref:{x/w:<frame>, y/h:<frame>}].
- [element or image_path]: [position/size with frames, e.g., 100% of slide area; or 40% of panel width]
- ... (cover key elements)

Applying principles (concise bullets):
1. GRAPHIC_COLOR: [palette/contrast]
2. GRAPHIC_POSITION: [proportions with frames]
3. GRAPHIC_SIZE: [ratios]
4. IMAGE_POSITION: [zones]
5. IMAGE_SIZE: [ratios]
6. TEXT_COLOR: [contrast/emphasis]
7. TEXT_POSITION: [divisions]
8. TEXT_SIZE: [hierarchy ratios; **legibility floor: all font sizes >= 10% of container height**] 
</think>

## STEP 3: COMPLETE DESIGN PLAN
<plan>
Write one paragraph in this order to state the complete design plan: background -> containers (e.g., panel) -> subtitle/header -> title -> main content (images/text) -> dividers/accents -> spacing/balance. Use simple language that a 10-year-old child can understand. Avoid ambiguous or fancy design terms. End with a single save path line if specified. Always include image file paths at first mention.
</plan>


Rules:
1) Use proportions ONLY (no px/pt). For any element placement, specify a bbox as x:.. y:.. w:.. h:.. with values in [0,1] AND specify the frame (default: of panel; otherwise of slide or a named frame).
 x/y refer to the TOP-LEFT corner. Do not use center-point notation.
2) **No Ranges:** The <plan> paragraph MUST use single, concrete percentages (e.g., 35%). **DO NOT** write ranges (e.g., DO NOT write '30-40%').    
3) **Legibility Floor: Ensure all text font sizes are at least 10% of their container's height.**
4) Omit minor, complex visual effects (like subtle textures or complex shadows).
5) Be specific about colors: Use common color names (e.g., 'white', 'black', 'warm gold') or specific RGB values. Do NOT use abstract theme keys (like 'accent_1') as this causes ambiguity.
6) Be specific about containters:  When introducing a container (like 'panel' or 'title_band'), name it using its ID (e.g., \...a container named **panel**\).
\end{tcblisting}

\begin{tcblisting}{
listing only,
breakable,
enhanced,
colback=lightblue!5,
colframe=lightblue!50,
rounded corners,
top=6pt,bottom=6pt,left=8pt,right=8pt,boxsep=4pt,
title=User Prompt for Plan Generation ({\name}),
listing options={style=promptstyle}
}
Generate a complete 3-step design plan (Analysis, Strategy, Plan) based on the user request below.
Adhere strictly to the 3-step format defined in your system prompt.

--- [USER REQUEST START] ---
{user_prompt}
--- [USER REQUEST END] ---
\end{tcblisting}

\UseRawInputEncoding

\lstdefinestyle{promptstyle}{
  basicstyle=\ttfamily\scriptsize,
  breaklines=true,
  breakatwhitespace=false,
  columns=fullflexible,
  keepspaces=true,
  showstringspaces=false,
  tabsize=2
}

\begin{tcblisting}{
listing only,
breakable,
enhanced,
colback=lightblue!5,
colframe=lightblue!50,
rounded corners,
top=1pt,bottom=1pt,left=2pt,right=2pt,boxsep=0.5pt,
title=System Prompt for Plan Revision ({\name}),
listing options={style=promptstyle}
}
You are an expert slide design plan editor. The **goal** is to improve a baseline slide plan into a higher-quality, more legible, and better-aligned plan by applying a feedback checklist, while still adapting the reference slide style.

IMPORTANT: This is an EDITING task, not a from-scratch design task.
- Treat the provided baseline plan as the starting point.
- Preserve everything by default.
- Apply only the required changes from the feedback checklist.
- Any aspect marked GOOD must remain unchanged.
- Priority rule: Use the checklist Revision Plan (KEEP/CHANGE) as the ONLY authority for deciding edits.
- Current Status is descriptive and may be approximate; do NOT use it to override the baseline plan.
- The checklist may be partial (not mentioning every element). Treat it as REQUIRED fixes, not a complete specification.
- After applying required fixes, do a quick non-regression check (no overlap/crowding, high contrast, clear hierarchy, legibility floor).
- The baseline plan may use an older writing style. First normalize it into the current required format (frames + bboxes x/y/w/h in [0,1]) before revising.


Your response must follow this EXACT format:

## STEP 1: EXTRACTING TRANSFERABLE PRINCIPLES
<analysis>
Reference aspect ratio: [16:9 or 4:3 - MUST preserve]
Visual System: [Core visual logic]
Transferable Elements: [Color harmony, spacing rhythm, hierarchy]
Context-Specific Elements: [What must change for this content]
Key Insight: [One guiding principle]
</analysis>

## STEP 2: INTELLIGENT ADAPTATION STRATEGY
<think>
The user wants: [Brief restatement]
The reference provides: [Content type + approach]

**Frames (declare once):**
- slide: full canvas (base frame)
- panel: bbox as ratios of slide (x,y,w,h in [0,1])
- [optional] other frames (title_band / column_left / ...), each with bbox relative to its parent
- Inheritance: child elements inherit the parent frame; only annotate a frame when it changes.

**NORMALIZE BASELINE (required):**
- The previous plan may use an older writing style. First, restate the BASELINE layout decisions using explicit bboxes.
- For each key element, specify: element id (e.g., title/body_text/img_1) + frame + bbox x:.. y:.. w:.. h:.. in [0,1].
- Key elements must include: panel (if used), title, main text block(s), and each image path.
- If the previous plan is ambiguous, choose a single reasonable concrete value consistent with the reference style. Do NOT use ranges.

**Image Role Analysis:**
- [image_path]: Role=[Background/Illustration/Logo/...]; Evidence=\[quote/paraphrase]\
- ... (Repeat per image)

**ASPECT EDIT MAP (required; cover all 8 aspects):**
- For each aspect below, write: KEEP or CHANGE -> affected element(s) -> exact edit (bbox/color/font scale).
- If aspect is GOOD in the checklist: write KEEP and explicitly state what will remain unchanged.
- If aspect needs revision: write CHANGE and specify the minimal targeted edits.

1. GRAPHIC_COLOR: ...
2. GRAPHIC_POSITION: ...
3. GRAPHIC_SIZE: ...
4. IMAGE_POSITION: ...
5. IMAGE_SIZE: ...
6. TEXT_COLOR: ...
7. TEXT_POSITION: ...
8. TEXT_SIZE: ...

**Non-regression checks (short):**
- No overlap or crowding between title/text/images.
- Text contrast is high against background.
- Clear hierarchy (title > body).
- Legibility floor satisfied.
</think>

## STEP 3: COMPLETE DESIGN PLAN
<plan>
Write one paragraph in this order to state the complete design plan: background -> containers (e.g., panel) -> subtitle/header -> title -> main content (images/text) -> dividers/accents -> spacing/balance. Use simple language that a 10-year-old child can understand. Avoid ambiguous or fancy design terms. End with a single save path line if specified. Always include image file paths at first mention.

IMPORTANT: This is a revised plan. Keep the overall structure of the baseline unless the checklist requires changes.
</plan>





Rules:
1) Use proportions ONLY (no px/pt). For any element placement, specify a bbox as x:.. y:.. w:.. h:.. with values in [0,1] AND specify the frame (default: of panel; otherwise of slide or a named frame).
 x/y refer to the TOP-LEFT corner. Do not use center-point notation.
2) **No Ranges:** The <plan> paragraph MUST use single, concrete percentages (e.g., 35%). **DO NOT** write ranges (e.g., DO NOT write '30-40%').    
3) **Legibility Floor: Ensure all text font sizes are at least 10% of their container's height.**
4) Omit minor, complex visual effects (like subtle textures or complex shadows).
5) Be specific about colors: Use common color names (e.g., 'white', 'black', 'warm gold') or specific RGB values. Do NOT use abstract theme keys (like 'accent_1') as this causes ambiguity.
6) Be specific about containters: When introducing a container (like 'panel' or 'title_band'), name it using its ID (e.g., \...a container named **panel**\).
\end{tcblisting}

\begin{tcblisting}{
listing only,
breakable,
enhanced,
colback=lightblue!5,
colframe=lightblue!50,
rounded corners,
top=6pt,bottom=6pt,left=8pt,right=8pt,boxsep=4pt,
title=User Prompt for Plan Revision ({\name}),
listing options={style=promptstyle}
}
You must **revise** a slide plan based on the provided feedback.
Your response MUST be a **new, complete 3-step plan** that integrates the required changes.

--- [USER REQUEST START] ---
{user_prompt}
--- [USER REQUEST END] ---

--- [BASELINE PLAN START] ---
{prev_plan}
--- [BASELINE PLAN END] ---

--- [FEEDBACK CHECKLIST START] ---
{comt}
--- [FEEDBACK CHECKLIST END] ---

**Your Task:**
1) Carefully read the [FEEDBACK CHECKLIST].
2) In STEP 2 (<think>), first NORMALIZE the baseline plan into explicit frames + bboxes (x,y,w,h in [0,1]).
3) Then apply ONLY the required changes from the checklist, preserving all items marked GOOD.
4) The checklist may be partial; do not redesign unrelated elements.
5) Output a new complete three-part plan (STEP 1/2/3). STEP 3 (<plan>) must include explicit bboxes.
\end{tcblisting}

\UseRawInputEncoding

\lstdefinestyle{promptstyle}{
  basicstyle=\ttfamily\scriptsize,
  breaklines=true,
  breakatwhitespace=false,
  columns=fullflexible,
  keepspaces=true,
  showstringspaces=false,
  tabsize=2
}

\begin{tcblisting}{
listing only,
breakable,
enhanced,
colback=lightblue!5,
colframe=lightblue!50,
rounded corners,
top=1pt,bottom=1pt,left=2pt,right=2pt,boxsep=0.5pt,
title=Prompt Template for VLM Judge ({\name}),
listing options={style=promptstyle}
}
You are an expert presentation designer and a strict visual aesthetics judge. Please evaluate the design quality of the provided slide based on the following criteria.

Here is the intended original content/prompt for this slide:
<user_input>
{user_input}
</user_input>

*CRITICAL INSTRUCTION*: If the <user_input> contains file paths, URLs, or instructions to insert specific images, please IGNORE checking whether those specific image contents are correctly rendered. Your focus is strictly on the text completeness, layout, readability, and overall visual aesthetics.

CRITERIA TO EVALUATE (Scale 0-10, where 10 is flawless and 0 is completely unusable):
{criteria_text}

You MUST respond strictly in valid JSON format. Your output must exactly match this JSON schema:
{{
    metrics: {{
        color_and_contrast: {{score: <float>, reasoning: <string>}},
        layout_and_alignment: {{score: <float>, reasoning: <string>}},
        typography: {{score: <float>, reasoning: <string>}},
        overall_aesthetics: {{score: <float>, reasoning: <string>}},
        faithfulness: {{score: <float>, reasoning: <string>}}
    }},
    actionable_suggestions: [
        <string: specific coordinate or element to fix>,
        <string: another specific suggestion>
    ]
}}
Be a critical and strict judge. Use the full 0-10 scale.
\end{tcblisting}

\begin{tcblisting}{
listing only,
breakable,
enhanced,
colback=lightblue!5,
colframe=lightblue!50,
rounded corners,
top=6pt,bottom=6pt,left=8pt,right=8pt,boxsep=4pt,
title=Evaluation Criteria for VLM Judge ({\name}),
listing options={style=promptstyle}
}
1. Faithfulness & Content Integrity: Evaluate the text from user instruction if it is fully readable.
2. Color Harmony & Contrast: Evaluate color palette harmony, background-to-text contrast, and readability.
3. Layout, Alignment & Spacing: Assess visual composition, alignment, margins, and whitespace balance.
4. Overall Professionalism: Evaluate if it looks like a high-quality, modern presentation for top-tier venues.
\end{tcblisting}

\end{document}